\documentclass[sigconf]{acmart}
\AtBeginDocument{%
  }

\usepackage{amsmath,amsthm}
\usepackage{algorithm} 
\usepackage{algpseudocode}
\usepackage{graphicx}
\usepackage{subcaption}
\usepackage{multirow}

\usepackage{longtable}  
\usepackage{booktabs}

\usepackage{booktabs}
\usepackage{threeparttable}
\setcopyright{acmlicensed}
\copyrightyear{2018}
\acmYear{2018}
\acmDOI{XXXXXXX.XXXXXXX}
\acmConference[Conference acronym 'XX]{Make sure to enter the correct
  conference title from your rights confirmation email}{June 03--05,
  2018}{Woodstock, NY}
\acmISBN{978-1-4503-XXXX-X/2018/06}

\begin{document}

\title{HERO: History-Enriched Rollout Training for Long-Horizon Autoregressive Neural Operators}

\settopmatter{authorsperrow=5}

\author{Jiaquan Zhang}
\affiliation{%
  \institution{School of Computer Science and Engineering, UESTC}
  \city{Chengdu}
  \country{China}
}

\author{Shuxu Chen}
\affiliation{%
  \institution{Electronics and Information Convergence Engineering, KHU}
  \city{Yongin-si}
  \country{Korea}}

\author{Haifan Meng}
\affiliation{%
  \institution{School of Computer Science and Engineering, UESTC}
  \city{Chengdu}
  \country{China}
}

\author{Yi Lu}
\affiliation{%
 \institution{Department of Mathematical Sciences, UOL}
 \city{Liverpool}
 \country{England}}

\author{Zhihan Lyu}
\affiliation{%
  \institution{School of Computer Science and Technology, XDU}
  \city{Shanxi}
  \country{China}}

\author{Fan Mo}
\affiliation{%
  \institution{School of Mechanical and Electrical Engineering, UESTC}
  \city{Chengdu}
  \country{China}}

\author{Wei Dong}
\affiliation{%
  \institution{College of Computer and Information Engineering, XAUAT}
  \city{Xi'an}
  \country{China}}

\author{Yang Yang}
\affiliation{%
  \institution{School of Computer Science and Engineering, UESTC}
  \city{Chengdu}
  \country{China}
}

\author{Chaoning~Zhang}
\affiliation{%
  \institution{School of Computer Science and Engineering, UESTC}
  \city{Chengdu}
  \country{China}
}

\renewcommand{\shortauthors}{Zhang et al.}

\begin{abstract}
Neural operators provide fast surrogates for time-dependent partial differential
equations (PDEs) by applying a learned evolution operator recursively to its own
predictions, but this autoregressive rollout feeds every prediction error back as
input, so local errors accumulate. 
Existing rollout-training strategies reduce the mismatch
between training inputs and self-generated states, yet their supervision still
measures only the absolute discrepancy from the ground-truth trajectory. 
Such supervision is therefore uninformative about whether the operator has overcome
the long-horizon failure behaviors it exhibited earlier during optimization. We
propose history-enriched rollout training (HERO), which augments conventional absolute trajectory supervision with relative supervision derived from the model’s optimization history.
HERO ranks detached candidate rollouts from a periodically
refreshed lagged operator, the current model, and a perturbed input by rollout
error, spectral discrepancy, energy drift, and error growth, and selects the
strongest failure trajectory as reference. This reference enters a margin-based
objective as a fixed comparison baseline, inducing a bounded, sample-dependent
reweighting of the ground-truth rollout gradient rather than an independent
gradient direction, which we further analyze theoretically. Experiments on nine
PDE benchmarks with spectral and attention-based backbones show that HERO
consistently improves long-horizon accuracy, stable rollout length, and
out-of-distribution robustness at no inference-time cost. These results indicate
that history-enriched relative supervision is effective for stabilizing
long-horizon autoregressive prediction.
\end{abstract}




\maketitle

\section{Introduction}
\label{sec:introduction}

During autoregressive inference, each prediction is fed back as input, repeatedly propagating local errors and shifting the rollout away from the training distribution~\cite{li2023scalable,McCabeHSB23}. 
Nonlinear, multiscale, and repeated spectral dynamics can amplify these deviations, causing spectral distortion and eventual rollout instability~\cite{McCabeHSB23,lippe2023pde,li2026sgno,jiang2026hierarchical}.
Low one-step error therefore does not guarantee accurate and stable long-horizon rollouts~\cite{McCabeHSB23,lippe2023pde}.
For chaotic systems, long-term statistical fidelity must be evaluated separately from pointwise one-step accuracy~\cite{li2022learning,SchiffWPHKSZ24}.
Existing approaches mitigate this mismatch by exposing the operator to self-generated states during training, including multi-step rollout training, push-forward strategies, differentiable-solver strategies, and recurrent operator training~\cite{BrandstetterWW22,um2020solver,ye2025recurrent}. 
These rollout-training strategies reduce the mismatch between training and inference inputs by modifying the states encountered during training, while their objectives remain grounded in discrepancies between predicted and reference trajectories rather than an explicitly history-enriched failure reference~\cite{BrandstetterWW22,um2020solver,ye2025recurrent}.
Such supervision measures how far the current rollout deviates from the target evolution, but does not explicitly reveal whether its long-horizon error growth or previously observed failure patterns persist~\cite{McCabeHSB23,list2025differentiability,lippe2023pde}.

To illustrate this limitation, we compare the standard Fourier neural operator
(FNO) with its push-forward variant (FNO-PF), which propagates model-generated
states into subsequent steps and updates the operator on the resulting rollout
errors. As shown in Figure~\ref{fig1}, FNO-PF delays the onset of degradation
relative to FNO but retains a similar long-horizon growth trend. Exposure to
self-generated states therefore improves long-horizon prediction, yet absolute
supervision still cannot indicate whether the operator has overcome the failure
behaviors it produced earlier during optimization.
The model's optimization history provides a natural source of comparative information for addressing this supervision gap. 
We treat parameter states saved at earlier optimization stages as historical checkpoints; each checkpoint defines a complete autoregressive operator whose long-horizon rollout is evaluated from the same initial conditions. 
Comparing these rollouts
against the final operator (Figure~\ref{fig:motivation_history}) shows that a
checkpoint's relative quality depends on the rollout horizon and does not
improve monotonically with training: some checkpoints stay competitive over the
initial steps but fall behind at longer horizons, whereas others differ in their
transition times and error-growth patterns. Historical checkpoints are therefore
not uniformly weaker copies of the final operator; each retains a distinct
long-horizon failure behavior.
Moreover, as the current operator improves, a fixed historical trajectory
becomes too easy to outperform, reducing the contribution of the relative
objective. These observations motivate adapting the reference to both the
rollout behavior and the current optimization stage.

We propose a history-enriched rollout training framework (HERO) that supplements conventional ground-truth trajectory regression with relative long-horizon supervision. 
Instead of relying on a fixed historical checkpoint, HERO dynamically constructs up to three detached candidate
trajectories from a periodically refreshed lagged operator, the current rollout, and an optional locally perturbed input. It evaluates their rollout error, spectral discrepancy, energy deviation, and error-growth
behavior, and selects an informative failure trajectory as the comparison reference. The current operator is then optimized both to match the ground-truth evolution and to outperform the selected reference through a margin-based objective.
The selected reference is used only as a detached comparison baseline and introduces no independent gradient direction. Instead, the relative objective induces a bounded, sample-dependent reweighting of the current ground-truth rollout gradient. We characterize this mechanism through a
covariance decomposition of the aggregated update direction and establish conditional finite-step reachability of a prescribed long-horizon risk
region under explicit regularity and alignment conditions. HERO is architecture agnostic, and all reference-construction components are used only during training, leaving the inference-time architecture and cost unchanged.

The main contributions are summarized as follows:
\textbf{\textit{(i)}} We propose HERO, an architecture-agnostic history-enriched rollout training framework that converts dynamically selected long-horizon failure behaviors into relative supervision while retaining conventional ground-truth trajectory regression.
\textbf{\textit{(ii)}} We characterize the optimization mechanism induced by stop-gradient reference trajectories. The reference applies a bounded, sample-dependent reweighting to the ground-truth rollout gradients rather than introducing an independent distillation direction. We further derive a covariance decomposition of the aggregated update direction and establish conditional finite-step reachability of a prescribed long-horizon risk region.
\textbf{\textit{(iii)}} We conduct experiments on nine PDE benchmarks spanning one-, two-, and three-dimensional dynamics with spectral and attention-based backbones. The results show that HERO consistently improves long-horizon rollout accuracy, stability, and out-of-distribution robustness over strong autoregressive-training baselines without increasing inference-time cost.

\section{Related Work}
\label{sec:related}

Rollout-based training reduces the train--test mismatch of autoregressive
neural operators~\cite{kovachki2023neural,li2020fourier} by exposing the model
to self-generated states, through push-forward and temporal
bundling~\cite{BrandstetterWW22}, differentiable solver-in-the-loop
correction~\cite{um2020solver}, recursive graph
simulators~\cite{sanchez2020learning,PfaffFSB21}, unrolling
taxonomies~\cite{koehler2024apebench,list2025differentiability}, and recurrent
operators~\cite{ye2025recurrent}; yet their supervision measures only the
absolute discrepancy from the ground-truth trajectory. A complementary
direction stabilizes long horizons through spectral and structural
design~\cite{McCabeHSB23,lippe2023pde,li2026sgno,jiang2026hierarchical,HuangG25}
or long-term statistical
constraints~\cite{li2022learning,jiang2023training,SchiffWPHKSZ24}, but modifies
the operator or objective rather than supervising against the model's own
failures. HERO addresses this complementary gap, leaving the inference-time
operator unchanged and using dynamically selected historical failure
trajectories to supervise whether the current rollout improves over previously
observed failure behavior. The complete introduction of related works is in Appendix~\ref{app:rl}.

\section{Motivation}
\label{sec:motivation}
This section presents two observations motivating HERO: prediction errors accumulate substantially over long autoregressive rollouts, and historical checkpoints exhibit distinct horizon-dependent behaviors. These findings motivate complementing absolute trajectory supervision with history-enriched relative supervision.

\subsection{Long-Horizon Error Accumulation}
Figure~\ref{fig1} compares FNO, FNO-PF, and HERO over
200-step free autoregressive rollouts from the same test initial conditions. 
The three methods remain relatively close during the initial interval, but their errors separate substantially as the rollout horizon increases. FNO exhibits rapid error accumulation. FNO-PF delays part of the degradation, yet retains a similar long-range growth trend. 
In contrast, HERO maintains lower error and slower amplification throughout
the medium- and long-horizon rollout.
Figure~\ref{fig1} shows the ground-truth and predicted fields at step 100 for a representative test trajectory. The FNO and FNO-PF predictions exhibit clear phase, amplitude, and local oscillatory errors, whereas HERO more closely preserves the spatial structure of the ground-truth field.
These results show that short-horizon accuracy and exposure to self-generated states do not by themselves eliminate long-horizon degradation. Absolute supervision measures the discrepancy between the current rollout and the ground truth, but does not explicitly indicate whether the operator has overcome failure behaviors produced earlier during optimization.

\begin{figure}[t]
    \centering
    \includegraphics[width=0.485\textwidth]{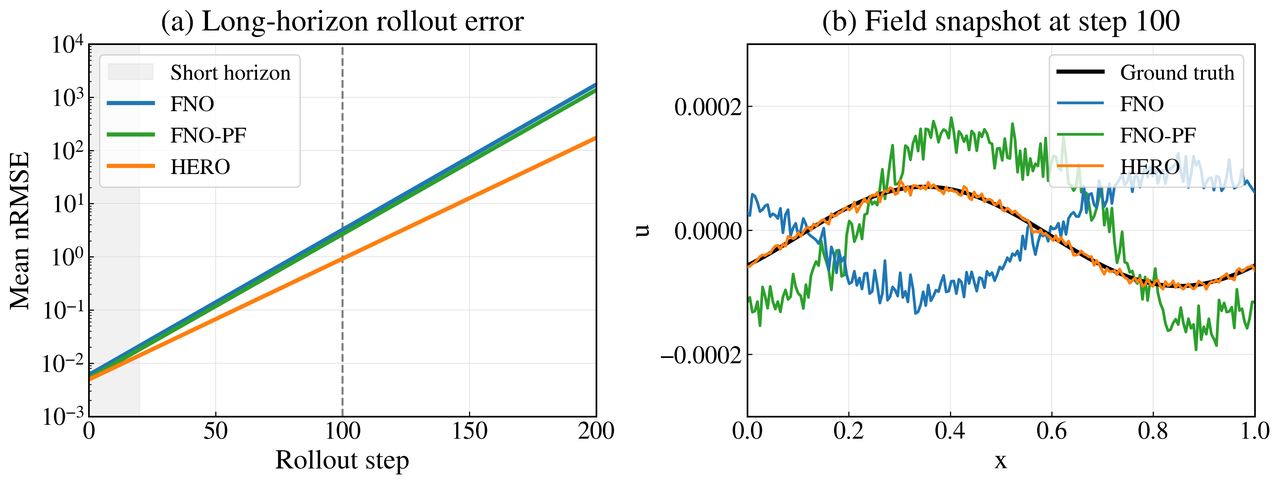}
    \caption{
    Short-horizon accuracy versus long-horizon rollout stability of FNO,
    FNO-PF, and HERO. (a) Mean normalized root-mean-square error (nRMSE) over 200-step free autoregressive rollouts from the same test initial conditions. (b) Ground-truth and predicted fields at step 100 for a representative test trajectory selected using a predefined criterion
    }
    \label{fig1}
\end{figure}

\begin{figure}[t]
    \centering
    \includegraphics[width=0.47\textwidth]{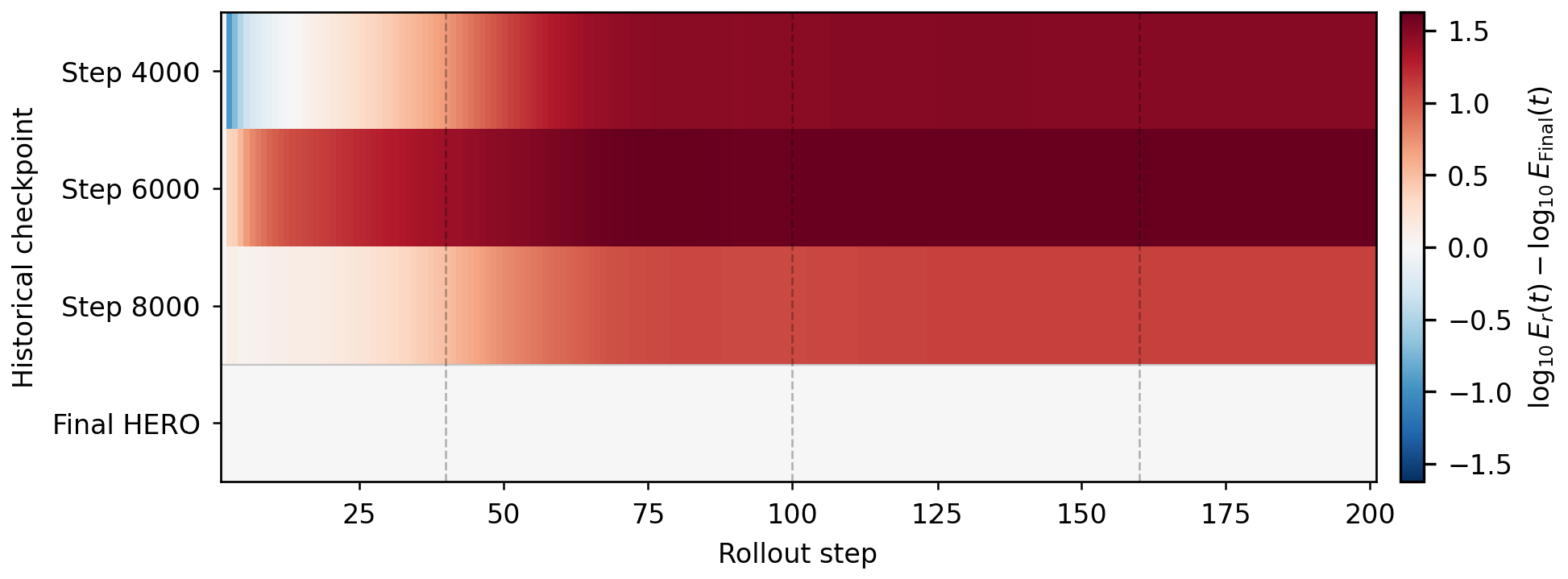}
    \caption{
    Relative rollout-error landscape across training checkpoints over 200-step free autoregressive rollouts. Each
row shows the base-10 logarithmic ratio between the rollout
error of a historical checkpoint and that of Final HERO. The
bottom row corresponds to Final HERO and is identically
zero. Dashed lines mark rollout steps 40, 100, and 160. 
    }
    \label{fig:motivation_history}
\end{figure}

\subsection{Distinct Rollout Behaviors across Training Checkpoints}

Each intermediate checkpoint defines a complete autoregressive operator
and therefore induces its own long-horizon rollout behavior.
Figure~\ref{fig:motivation_history} compares checkpoints saved at training
steps 4000, 6000, and 8000 with Final HERO. All checkpoints are evaluated
from the same test initial conditions using 200-step free autoregressive rollouts.

For checkpoint $s$, the mean relative error at rollout step $t$ is
\begin{equation}
E_s(t)
=
\frac{1}{N_{\mathrm{test}}}
\sum_{i=1}^{N_{\mathrm{test}}}
\frac{
\left\|
\hat{u}_{s,t}^{(i)}-u_t^{(i)}
\right\|_2
}{
\left\|
u_t^{(i)}
\right\|_2+\varepsilon
},
\label{eq:motivation_error}
\end{equation}
where $N_{\mathrm{test}}$ is the number of test trajectories,
$\hat{u}_{s,t}^{(i)}$ and $u_t^{(i)}$ are the predicted and ground-truth
states for trajectory $i$ at rollout step $t$, respectively, and
$\varepsilon>0$ is a numerical-stability constant. We compare each
checkpoint with Final HERO through the base-10 logarithmic ratio
\begin{equation}
\Delta_s(t)
=
\log_{10}
\left(
\frac{
E_s(t)+\varepsilon
}{
E_{\mathrm{Final}}(t)+\varepsilon
}
\right),
\label{eq:motivation_error_landscape}
\end{equation}
where $E_{\mathrm{Final}}(t)$ is the corresponding error of Final HERO.
Positive values of $\Delta_s(t)$ indicate a larger rollout error than
Final HERO, negative values indicate a smaller error, and values near zero
indicate comparable errors.
Stacking $\Delta_s(t)$ across checkpoints yields the relative error
landscape in Figure~\ref{fig:motivation_history}. The horizontal axis
represents the rollout step, while the vertical axis represents the
training checkpoint. The color scale is centered at zero, and the Final-HERO row is identically zero.
The landscape shows that the relative quality of a checkpoint depends on
the rollout horizon and does not improve monotonically throughout
training. Some historical checkpoints remain competitive during the
initial rollout steps but become substantially less accurate as the
rollout continues. Different checkpoints also exhibit distinct transition
times and error-growth patterns. Historical checkpoints are therefore not
merely uniformly weaker versions of the final model; they retain different
long-horizon error-propagation behaviors.

These observations motivate dynamic history-enriched supervision.
Historical rollouts can provide informative failure references, but a
fixed checkpoint can become stale, whereas an excessively frequent refresh
provides little contrast. HERO therefore maintains a controlled training
lag, evaluates candidate rollouts using long-horizon diagnostics, and uses
the selected trajectory to complement conventional absolute supervision.

\begin{figure*}
    \centering
    \includegraphics[width=0.87\linewidth]{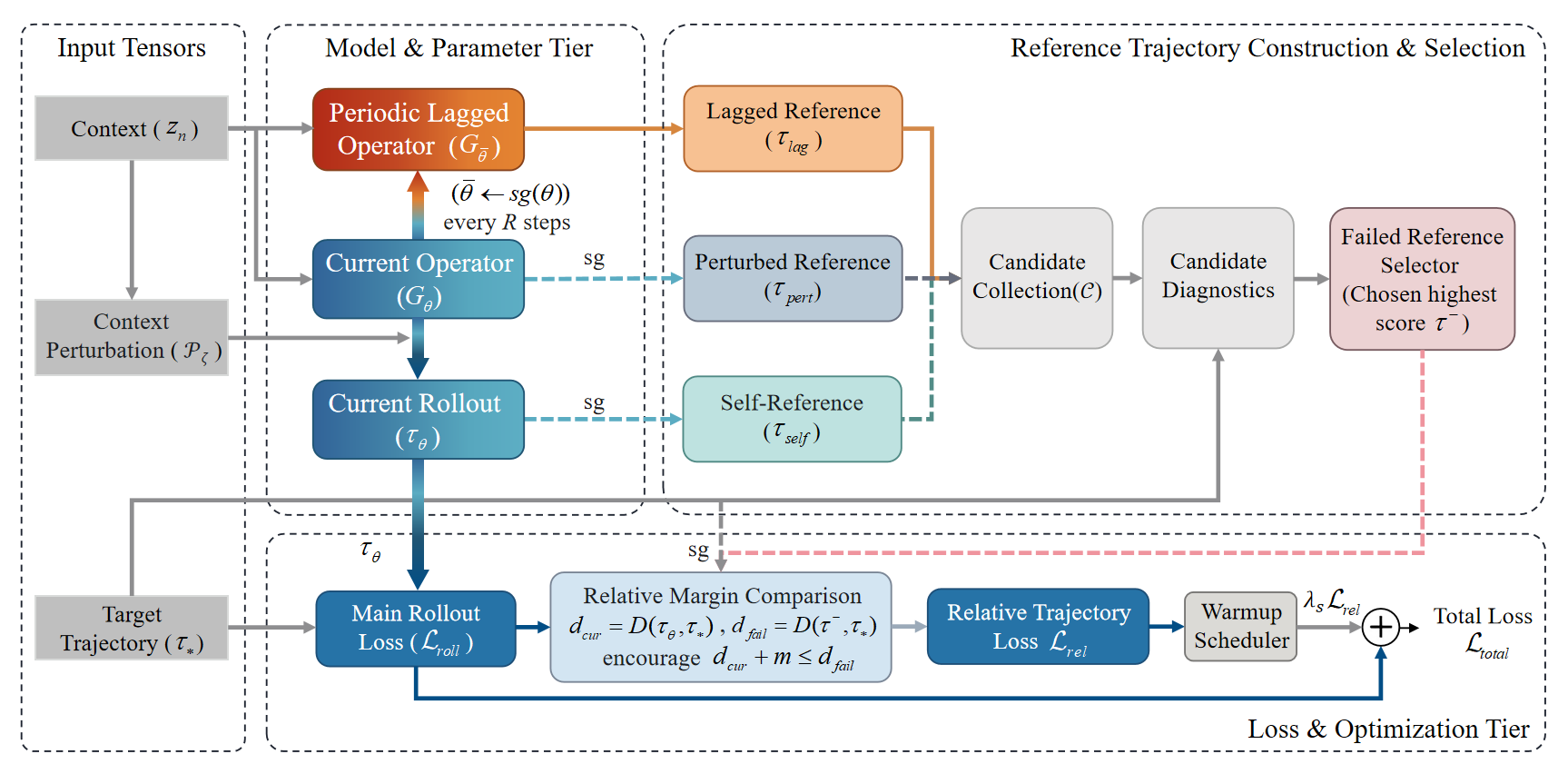}
    \caption{Overview of the HERO training framework. A lagged operator, refreshed every $R$ steps, supplies one of three detached candidate rollouts, ranked by rollout, spectral, energy, and error-growth diagnostics to select the failure reference $\tau^-$. The absolute rollout loss is combined with a margin objective driving the current trajectory error below the reference error.}
    \label{fig:main}
\end{figure*}

\section{Method}
\label{sec:method}
Figure~\ref{fig:main} gives an overview of HERO. 
From a shared input context, the current operator produces the rollout to be optimized while the lagged and perturbed branches produce detached candidates; the diagnostics then
rank these candidates into a failure reference, which meets the current rollout in a margin comparison that is added to the absolute rollout loss.

\subsection{Problem Formulation}

Consider a class of time-dependent PDEs defined on a spatial domain $\Omega\subseteq\mathbb{R}^{d}$. The state variable is denoted by $u(\mathbf{x},t)\in\mathbb{R}^{c_u}$ for
$\mathbf{x}\in\Omega$, where $c_u$ is the number of physical channels.
At the discrete time point $t_n=n\Delta t$, we denote the spatial state
by $u_n=u(\cdot,t_n)$.

Let $H$ denote the number of most recent states provided to the operator
at each prediction step. The input context at time step $n$ is defined as
$z_n=(u_{n-H+1},\ldots,u_n;\mu)$, where $\mu$ denotes conditioning
information that may influence the underlying dynamics, such as equation
parameters, external forcing, boundary conditions, or geometric
information. We consider a differentiable time-stepping operator
$G_\theta$, parameterized by trainable parameters $\theta$, that maps the
current context to a one-step prediction, namely
$G_\theta:z_n\mapsto\hat{u}_{n+1}$.
To generate a length-$K$ autoregressive rollout, the operator is
recursively applied to its own predictions. For notational consistency,
we initialize $\hat{u}_j=u_j$ for $j=n-H+1,\ldots,n$. For each rollout
step $k=1,\ldots,K$, the subsequent state is generated as
\begin{equation}
\hat{u}_{n+k}
=
G_\theta
\left(
\hat{u}_{n+k-H},
\ldots,
\hat{u}_{n+k-1};
\mu
\right).
\end{equation}
The predicted trajectory is denoted by
$\tau_\theta=(\hat{u}_{n+1},\ldots,\hat{u}_{n+K})$, and the corresponding
ground-truth trajectory is denoted by
$\tau_\star=(u_{n+1},\ldots,u_{n+K})$.

HERO is agnostic to the internal architecture of $G_\theta$. It only
requires the base operator to support differentiable one-step prediction
and recursive rollout. HERO therefore modifies the training procedure
without changing the inference-time architecture.

\subsection{Absolute Rollout Supervision}

Given the predicted trajectory $\tau_\theta$ and the corresponding
ground-truth trajectory $\tau_\star$, we first define a trajectory-level
discrepancy. This discrepancy serves as both the conventional absolute
supervision and a common measure for evaluating the current and candidate
trajectories against the same ground-truth trajectory.

Let $\mathcal{X}$ denote the discrete state space induced by the chosen
spatial discretization. For two states $v,w\in\mathcal{X}$, we define the
normalized state error as
\begin{equation}
d_{\mathcal{X}}(v,w)
=
\frac{
\|v-w\|_{\mathcal{X}}
}{
\|w\|_{\mathcal{X}}+\varepsilon
},
\end{equation}
where $\varepsilon>0$ is a constant for numerical stability. The spatial
norm is defined as
\begin{equation}
\|v\|_{\mathcal{X}}^2
=
\sum_{i=1}^{N_x}
\omega_i
\|v(\mathbf{x}_i)\|_2^2.
\end{equation}
Here, $N_x$ is the number of spatial sampling points $\mathbf{x}_i$,
$\|\cdot\|_2$ denotes the Euclidean norm over the $c_u$ physical
channels, and $\omega_i\geq0$ is the discrete integration weight
associated with the $i$-th spatial point. On a regular grid, the weights
are uniform. On irregular grids, finite-element meshes, or point clouds,
they may represent cell volumes, nodal areas, or quadrature weights.
For any length-$K$ trajectory $\tau=(v_1,\ldots,v_K)$, we define its
distance from the ground-truth trajectory as
\begin{equation}
D(\tau,\tau_\star)
=
\frac{1}{K}
\sum_{k=1}^{K}
d_{\mathcal{X}}
\left(
v_k,
u_{n+k}
\right).
\end{equation}
Applying this distance to the current rollout gives the absolute rollout
objective
\begin{equation}
\mathcal{L}_{\mathrm{roll}}
=
\mathbb{E}_{(z_n,\tau_\star)}
\left[
D(\tau_\theta,\tau_\star)
\right].
\label{eq:roll}
\end{equation}
The expectation is taken over sampled input contexts $z_n$ and their
corresponding ground-truth trajectories $\tau_\star$.

The objective $\mathcal{L}_{\mathrm{roll}}$ anchors the current rollout
to the reference PDE evolution. However, it evaluates the current
trajectory only through its absolute discrepancy from $\tau_\star$ and
does not explicitly measure improvement over structured long-horizon
behaviors produced during optimization. HERO therefore supplements the
absolute objective with history-enriched relative supervision.

\subsection{History-Enriched Relative Supervision}
\label{subsec:relative}

To provide the missing comparative signal, HERO compares the ground-truth
error of the current rollout with that of a dynamically selected,
model-generated failure trajectory.

Let $s$ denote the optimization iteration and let $\theta_s$ denote the
current parameters before the parameter update at iteration $s$. Relative
reference construction starts at iteration $s_0$. For a sampled context
$z_n$, the iteration-indexed current rollout is defined as
\begin{equation}
\tau_{\theta_s}
=
\operatorname{Rollout}
\left(
G_{\theta_s},
z_n,
K
\right).
\label{eq:current_rollout}
\end{equation}
Here, $\operatorname{Rollout}(G,z,K)$ denotes the length-$K$
autoregressive rollout of operator $G$ initialized from context $z$.
Thus, $\tau_{\theta_s}$ is the optimization-iteration-specific form of
the trajectory $\tau_\theta$ defined above. The dependence of subsequent
trajectory notation on $s$ and $n$ is omitted when no ambiguity arises.

HERO constructs the relative supervision in five stages:
1) updating a lagged operator when the refresh condition is satisfied,
2) generating candidate rollouts,
3) computing trajectory diagnostics,
4) scoring and selecting a failure trajectory, and
5) constructing the relative objective.

The lagged operator is updated periodically, whereas the remaining stages
are performed at each relative-training iteration. These stages correspond
to the main operations in Algorithm~\ref{alg:hero}.

\textit{1) Lagged Operator:}
HERO maintains the current operator $G_{\theta_s}$ and a lagged operator
$G_{\bar{\theta}_s}$. The current operator is updated by gradient descent,
whereas the lagged operator is periodically copied from the current
operator and remains frozen between refreshes.

Let $R$ be a positive integer denoting the refresh interval. For
$s\geq s_0$, the lagged parameters are updated as
\begin{equation}
\bar{\theta}_s
=
\begin{cases}
\operatorname{sg}(\theta_s),
& s=s_0, \\[2pt]
\operatorname{sg}(\theta_s),
& s>s_0 \ \text{and}\ (s-s_0)\bmod R=0, \\[2pt]
\bar{\theta}_{s-1},
& \text{otherwise},
\end{cases}
\label{eq:lag_update}
\end{equation}
where $\operatorname{sg}(\cdot)$ denotes the stop-gradient operator.
After each refresh, $G_{\bar{\theta}_s}$ remains frozen until the next
refresh and receives no backpropagated gradient.

Given the same input context $z_n$, the lagged operator generates
\begin{equation}
\tau_{\mathrm{lag}}
=
\operatorname{Rollout}
\left(
G_{\bar{\theta}_s},
z_n,
K
\right).
\label{eq:lag_rollout}
\end{equation}
The lagged rollout is used only as a comparison reference. HERO does not
train the current operator to reproduce or imitate
$\tau_{\mathrm{lag}}$. Updating the lagged operator at every iteration
would make the current and lagged rollouts nearly identical, whereas
keeping it fixed throughout training would produce an increasingly stale
reference. Periodic refresh maintains a controlled optimization delay
between the two operators.

\textit{2) Candidate Rollouts:}
HERO constructs up to three candidate rollouts: a detached version of the
current rollout, a rollout from the lagged operator, and a rollout
generated by the current operator from a locally perturbed input context.

(i) The self-reference is obtained by detaching the current rollout:
\begin{equation}
\tau_{\mathrm{self}}
=
\operatorname{sg}
\left(
\tau_{\theta_s}
\right).
\label{eq:self_rollout}
\end{equation}
For a deterministic base operator, $\tau_{\mathrm{self}}$ and
$\tau_{\theta_s}$ are numerically identical, but gradients propagate only
through $\tau_{\theta_s}$. The self-reference therefore provides a fixed
representation of the current rollout for comparison.
(ii) The lagged candidate is the rollout $\tau_{\mathrm{lag}}$ generated by
the frozen lagged operator. Since the lagged operator receives no
backpropagated gradient, $\tau_{\mathrm{lag}}$ provides a detached
reference from an earlier optimization state.
(iii) The perturbed candidate evaluates the current operator near the original
input context. Let $\zeta\sim p_\zeta$ denote a random perturbation
variable. The operator $\mathcal{P}_\zeta$ perturbs only the state-history
component of $z_n$ and leaves the conditioning information $\mu$
unchanged. The perturbed rollout is
\begin{equation}
\tau_{\mathrm{pert}}
=
\operatorname{sg}
\left[
\operatorname{Rollout}
\left(
G_{\theta_s},
\mathcal{P}_\zeta(z_n),
K
\right)
\right].
\label{eq:pert_rollout}
\end{equation}
The perturbation preserves the basic scale and boundary structure of the
input so that $\mathcal{P}_\zeta(z_n)$ remains in a local neighborhood of
the original context. On regular grids, $\mathcal{P}_\zeta$ may be
implemented as a weak spectral perturbation. On irregular grids or point
clouds, it may use local smoothing, neighborhood jitter, or a
graph-spectral perturbation.

When all three branches are enabled, the active candidate set is
\begin{equation}
\mathcal{C}
=
\left\{
\tau_{\mathrm{lag}},
\tau_{\mathrm{pert}},
\tau_{\mathrm{self}}
\right\}.
\label{eq:candidate_set}
\end{equation}
If no suitable perturbation is available, $\tau_{\mathrm{pert}}$ is
removed from $\mathcal{C}$. The number of active candidates is
$M=|\mathcal{C}|$, with $M\geq2$ because the lagged and self-reference
branches remain active.
For subsequent evaluation, we fix a deterministic enumeration of the
active candidates and denote them by
$\tau^{(1)},\ldots,\tau^{(M)}$. The enumeration follows the priority
$\tau_{\mathrm{lag}}\succ\tau_{\mathrm{pert}}\succ\tau_{\mathrm{self}}$,
with unavailable branches omitted. All active candidates are detached
from the computation graph and are used only for reference selection.

\textit{3) Trajectory Diagnostics:}
This stage evaluates the long-horizon behavior of each candidate and
determines which trajectory provides the strongest measured failure
reference. For each candidate $j=1,\ldots,M$, its length-$K$
trajectory is $\tau^{(j)}=(v_1^{(j)},\ldots,v_K^{(j)})$.

HERO computes scalar diagnostics that measure the overall rollout error,
spectral-amplitude mismatch, quadratic-energy drift, and temporal error
growth. These quantities are used only to rank the candidates. The same
numerical-stability constant $\varepsilon>0$ is used throughout.

The overall rollout error is
\begin{equation}
E_{\mathrm{roll}}^{(j)}
=
D
\left(
\tau^{(j)},
\tau_\star
\right).
\label{eq:diag_roll}
\end{equation}
This diagnostic measures the average state-space discrepancy between the
candidate trajectory and the ground-truth trajectory over the rollout
horizon.
When a spatial spectral transform is available, the spectral-amplitude
discrepancy is
\begin{equation}
E_{\mathrm{spec}}^{(j)}
=
\frac{1}{K}
\sum_{k=1}^{K}
\frac{
\left\|
\left|\mathcal{F}[v_k^{(j)}]\right|
-
\left|\mathcal{F}[u_{n+k}]\right|
\right\|_1
}{
\left\|
\left|\mathcal{F}[u_{n+k}]\right|
\right\|_1
+
\varepsilon
}.
\label{eq:diag_spec}
\end{equation}
Here, $\mathcal{F}$ denotes the spatial spectral transform applied
independently to each physical channel, $|\cdot|$ denotes the elementwise
magnitude, and $\|\cdot\|_1$ is taken over all spectral coefficients and
physical channels. This diagnostic measures whether the candidate assigns
incorrect amplitudes to the spatial modes. If no suitable spectral
transform is available, this diagnostic is omitted or replaced by one
defined in an appropriate spectral basis.

To measure global amplitude drift, HERO uses the quadratic state-energy proxy
\begin{equation}
\mathcal{E}(u)
=
\sum_{i=1}^{N_x}
\omega_i
\left\|
u(\mathbf{x}_i)
\right\|_2^2.
\label{eq:energy_proxy}
\end{equation}
This quantity equals the squared spatial norm defined above and is used as
an energy proxy rather than a PDE-specific conserved energy. The
corresponding energy drift is
\begin{equation}
E_{\mathrm{energy}}^{(j)}
=
\frac{1}{K}
\sum_{k=1}^{K}
\frac{
\left|
\mathcal{E}(v_k^{(j)})
-
\mathcal{E}(u_{n+k})
\right|
}{
\mathcal{E}(u_{n+k})
+
\varepsilon
}.
\label{eq:diag_energy}
\end{equation}
This diagnostic measures global amplitude loss, excessive dissipation,
and abnormal energy growth.
Let $e_k^{(j)}=d_{\mathcal{X}}(v_k^{(j)},u_{n+k})$
denote the state error of candidate $j$ at rollout step $k$. 
For $K\geq2$, the error-growth diagnostic is
\begin{equation}
E_{\mathrm{growth}}^{(j)}
=
\frac{1}{K-1}
\sum_{k=1}^{K-1}
\max
\left(
0,
e_{k+1}^{(j)}-e_k^{(j)}
\right).
\label{eq:diag_growth}
\end{equation}
This diagnostic accumulates only positive error increments and therefore
measures how strongly the prediction error grows along the rollout.
All candidate trajectories are detached before these diagnostics are
computed. Consequently, the diagnostic values are used only for candidate
ranking and reference selection, and no gradient is propagated through
them.

\textit{4) Candidate Scoring and Selection:}
HERO uses the trajectory diagnostics to score the active candidates and
select the candidate that exhibits the strongest measured long-horizon
failure.
Let $\mathcal{D}$ denote the set of active diagnostics. When all four
diagnostics are available, it is defined as
\begin{equation}
\mathcal{D}
=
\left\{
\mathrm{roll},
\mathrm{spec},
\mathrm{energy},
\mathrm{growth}
\right\}.
\end{equation}
Diagnostics that are not applicable to the current dataset are removed
from $\mathcal{D}$.

Because the diagnostics have different numerical scales, HERO normalizes
each diagnostic across the $M$ active candidates within the same training
sample. For each $q\in\mathcal{D}$ and $j=1,\ldots,M$, the normalized
diagnostic is
\begin{equation}
\widetilde{E}_q^{(j)}
=
\frac{
E_q^{(j)}
-
\min_{1\leq\ell\leq M}
E_q^{(\ell)}
}{
\max_{1\leq\ell\leq M}
E_q^{(\ell)}
-
\min_{1\leq\ell\leq M}
E_q^{(\ell)}
+
\varepsilon
}.
\label{eq:diag_normalization}
\end{equation}
This normalization maps each diagnostic to $[0,1]$ across the active
candidates and avoids introducing separate diagnostic-weighting
hyperparameters.
%
HERO computes the score of candidate $j$ as the average normalized diagnostic,
$S^{(j)}=\frac{1}{|\mathcal{D}|}\sum_{q\in\mathcal{D}}
\widetilde{E}_q^{(j)}$
\refstepcounter{equation}\label{eq:candidate_score}
\textup{(\theequation)}.
A larger score indicates stronger failure behavior across the active diagnostics.
HERO selects the highest-scoring candidate, indexed by
$j^\star=\arg\max_{1\leq j\leq M}S^{(j)}$. If multiple candidates attain
the same maximum, HERO selects the first one according to the fixed
candidate order defined above. The selected candidate is used as the
failure reference:
\begin{equation}
\tau^{-}
=
\tau^{(j^\star)}.
\end{equation}
The selected $\tau^{-}$ is therefore a model-generated trajectory that
exhibits the strongest measured failure behavior among the active
candidates. The diagnostics determine which candidate is selected,
whereas the relative objective compares the selected trajectory with the
current rollout using the common trajectory distance $D$. The diagnostic
values and candidate scores are not backpropagated.

\textit{5) Relative Objective:}
HERO converts the selected failure trajectory $\tau^{-}$ into a relative
improvement signal. The current trajectory error is
$d_{\mathrm{cur}}=D(\tau_{\theta_s},\tau_\star)$, while the selected
reference error is
$d_{\mathrm{fail}}=\operatorname{sg}[D(\tau^{-},\tau_\star)]$. The
stop-gradient operation treats the selected trajectory as a fixed
comparison reference, so gradients propagate only through the current
rollout.

HERO encourages the current rollout to outperform the selected reference
by a margin $m>0$, corresponding to
$d_{\mathrm{cur}}+m\leq d_{\mathrm{fail}}$. We implement this requirement
using the smooth margin loss
\begin{equation}
\ell_{\mathrm{rel}}
=
\frac{1}{\beta}
\log
\left(
1+
\exp
\left(
\beta
\left[
d_{\mathrm{cur}}
-
d_{\mathrm{fail}}
+
m
\right]
\right)
\right).
\label{eq:relative_loss}
\end{equation}
where $\beta>0$ controls the sharpness of the transition around the
margin. The loss assigns a larger penalty when the current rollout has
not outperformed the reference by the required margin and decreases as
the margin is satisfied.
Averaging the sample-level loss over the training distribution gives the
relative supervision objective
\begin{equation}
\mathcal{L}_{\mathrm{rel}}
=
\mathbb{E}_{(z_n,\tau_\star)}
\left[
\ell_{\mathrm{rel}}
\right].
\label{eq:rel}
\end{equation}
When the perturbed branch is enabled, the expectation also includes the
randomness of $\zeta$.

\subsection{Training Objective}
\label{subsec:total}

HERO introduces relative supervision only after the operator has learned
a basic time-stepping model. Reference construction is enabled at
iteration $s_0$, while the relative-supervision weight increases from
zero over a warmup period.
Let $W\in\mathbb{N}_{+}$ denote the warmup length and let
$\lambda_{\max}\geq0$ denote the maximum relative-supervision weight.
At optimization iteration $s$, the weight is
\begin{equation}
\lambda_s
=
\lambda_{\max}
\operatorname{clip}
\left(
\frac{s-s_0}{W},
0,
1
\right),
\end{equation}
where
$\operatorname{clip}(x,a,b)=\min(\max(x,a),b)$.
Thus, $\lambda_s=0$ for $s\leq s_0$, increases from zero to
$\lambda_{\max}$ over the following $W$ iterations, and remains equal to
$\lambda_{\max}$ afterward.
The complete training objective is
\begin{equation}
\mathcal{L}_{\mathrm{total}}
=
\mathcal{L}_{\mathrm{roll}}
+
\lambda_s
\mathcal{L}_{\mathrm{rel}}.
\label{eq:total}
\end{equation}
The absolute term learns the reference PDE evolution, while the relative
term encourages the current rollout to improve upon the selected failure
trajectory.

\section{Experiments}
\subsection{Experimental Setup}
\label{subsec:setup}

\textbf{Backbones and baselines.}
We build every method on two backbones: \emph{FNO}, a spectral neural operator,
and \emph{Transolver}, an attention-based neural operator with Physics-Attention~\cite{wu2024Transolver}.
On each backbone we compare one-step supervised training, push-forward training
(PF), PDE-Refiner (Refine), recurrent operator training (RNO), and HERO, denoting variants by \emph{Backbone}-\emph{Strategy} (e.g., FNO-PF)~\cite{BrandstetterWW22,lippe2023pde,ye2025recurrent}. 
PF and RNO are closest to HERO, since all three train on self-generated states yet supervise only the absolute discrepancy from the ground truth, whereas Refine corrects the output itself; HERO keeps ground-truth rollout regression and adds relative supervision against a dynamically history-enriched failure reference.
Full baseline descriptions are given in Appendix~\ref{app:baselines}.

\textbf{Benchmarks.}
We evaluate on nine PDE benchmarks spanning 1D, 2D, and 3D dynamics
(Table~\ref{tab:benchmarks_appendix}); APEBench supplies the task definitions, while PDEBench provides complementary scientific-machine-learning benchmark coverage~\cite{koehler2024apebench,takamoto2024pdebenchextensivebenchmarkscientific}: 
Dispersion~\cite{Korteweg01051895}, Burgers~\cite{BURGERS1948171}, KdV~\cite{Korteweg01051895}, and Kuramoto--Sivashinsky (KS)~\cite{SIVASHINSKY19771177,10.1143/PTPS.64.346} in 1D; Anisotropic Diffusion~\cite{fick1855ueber}, Kolmogorov Flow~\cite{MESHALKIN19611700}, and Decaying Turbulence (Navier--Stokes)~\cite{navier1823memoire} in 2D; and Swift Hohenberg~\cite{PhysRevA.15.319} and Unbalanced Advection in 3D~\cite{Courant1928Ber}.
Together they cover dispersion, nonlinear transport, dissipative chaos, forced and decaying turbulence, pattern formation, and multi-directional advection, stressing phase, spectral, energy, and stability behavior over long rollouts. 
All benchmarks use periodic boundaries and a single scalar channel, with $50$ training and $30$
test trajectories and $200$-step free autoregressive rollouts. Governing equations and coefficients are given in Appendix~\ref{app:benchmarks}.

\textbf{Metrics.}
For evaluation trajectory $j$ and rollout step $t$, the per-step normalized RMSE is
$\mathrm{nRMSE}_t^{(j)}=\|\hat u_t^{(j)}-u_t^{(j)}\|_2/(\|u_t^{(j)}\|_2+\varepsilon)$,
with $\varepsilon>0$ a numerical-stability constant. We report the set-averaged
$\mathrm{nRMSE@}1$ and $\mathrm{nRMSE@}100$ (one-step and $100$-step accuracy;
lower is better), the geometric-mean rollout error (GM$_{100}$)
\begin{equation}
  \mathrm{GM100}^{(j)}
    = \exp\!\Big(\tfrac{1}{100}\textstyle\sum_{t=1}^{100}
      \log\big(\mathrm{nRMSE}_t^{(j)}+\varepsilon\big)\Big),
  \label{eq:GM100}
\end{equation}
reported as its validation- and test-set means (Val/Test mean GM100; lower is
better), and the mean stable step
$\tfrac{1}{M}\sum_{j}S^{(j)}$ with
$S^{(j)}=\max\{T\in[1,100]\mid \mathrm{nRMSE}_t^{(j)}\le\tau\ \ \forall\,t\le T\}$
and $\tau=0.1$ (higher is better). Val mean GM100 is used for checkpoint selection. The detailed metrics descriptions are given in Appendix~\ref{app:metrics}

\textbf{Implementation.}
Each task provides 50 training trajectories of length 51 and 30 test trajectories of length 201; at test time every model rolls out $200$
steps from the first frame. All methods share the backbone, data split, optimizer,
and evaluation protocol, and use a $K_{\mathrm{pf}}=5$ push-forward base; HERO adds
the relative objective $\mathcal{L}_{\mathrm{rel}}$, giving
$\mathcal{L}_{\mathrm{total}}=\mathcal{L}_{\mathrm{roll}}+\lambda_s\mathcal{L}_{\mathrm{rel}}$.
Across all tasks HERO uses $K=5$, $s_0=4000$, $R=2000$, $W=2000$,
$\lambda_{\max}=0.02$, $m=0.02$, $\beta=5.0$, and candidate set
$\mathcal{C}=\{\tau_{\mathrm{lag}},\tau_{\mathrm{pert}},\tau_{\mathrm{self}}\}$;
reference construction is confined to training and adds no inference-time cost. We
optimize with AdamW (learning rate $10^{-3}$, warm-up--cosine schedule, $2000$
warm-up steps, $10\,000$ total steps, batch size $20$), and report test results at
the checkpoint with the best Val mean GM100. Backbone architectures per
dimension and further details are given in Appendix~\ref{app:impl}.
We fix the start step $s_0=4000$ and the refresh interval $R=2000$ of
\eqref{eq:lag_update} for all tasks: this pair introduces relative supervision
once the backbone has acquired a stable one-step evolution while leaving $6000$
steps of history-enriched reweighting, and it refreshes the lagged operator
$\bar{\theta}_s$ often enough to avoid a stale reference without collapsing it
onto the current operator $\theta_s$. A sensitivity analysis of both
hyperparameters is given in Appendix~\ref{subsec:hyper}.

\subsection{Main Results}
Table~\ref{tab:main} reports the mean performance on six PDE benchmarks
spanning one-, two-, and three-dimensional systems. The complete results
on all nine benchmarks, including the remaining three benchmarks and the
corresponding standard deviations over five seeds, are provided in
Tables~\ref{tab:full1d} and~\ref{tab:full23d} in Appendix~\ref{app:addexp}.
Across both the main and additional benchmarks, HERO consistently improves
FNO and Transolver. On the six main benchmarks, HERO achieves the best
result within each backbone family for all reported metrics.
On FNO, HERO reduces nRMSE@100 from $0.541$ to $0.153$ on
Burgers, from $1.131$ to $0.640$ on KS, from $1.868$ to $0.521$
on anisotropic diffusion, and from $13.629$ to $4.101$ on NS.
Compared with the strongest non-HERO strategy, it reduces
nRMSE@100 by $33.7\%$--$35.8\%$ across the six benchmarks.
Similarly, on Transolver, HERO reduces nRMSE@100 by
$33.6\%$--$36.7\%$ relative to the strongest competing strategy.
The improvements in nRMSE@100 and Stable step are substantially
larger than those in nRMSE@1. For example, HERO increases the
stable rollout length from $35.6$ to $81.7$ steps on Burgers and
from $8.2$ to $26.6$ steps on anisotropic diffusion. These results
indicate that HERO primarily suppresses long-horizon error
accumulation rather than merely improving one-step fitting, while
its consistent gains on both backbones demonstrate architectural
generality.
Because reference construction is confined to training, HERO leaves the
deployed operator identical to its backbone; a measured comparison of
parameter count, inference latency, memory, and training throughput is
reported in Appendix~\ref{subsec:cost}.

\begin{table}[t]
  \centering
  \caption{
    Main results on six PDE benchmarks.
  }
  \label{tab:main}
  \scriptsize
  \renewcommand{\arraystretch}{0.9}

  \begin{tabular*}{\columnwidth}{
    @{\extracolsep{\fill}}llllll@{}
  }
    \toprule
    Method
    & \shortstack{Val\\GM$_{100}$ $\downarrow$}
    & \shortstack{Test\\GM$_{100}$ $\downarrow$}
    & \shortstack{nRMSE\\@1 $\downarrow$}
    & \shortstack{nRMSE\\@100 $\downarrow$}
    & \shortstack{Stable\\step $\uparrow$} \\
    \midrule

    \multicolumn{6}{c}{\textit{Burgers (1D)}} \\
    \midrule
    FNO
    & 0.473 & 0.526 & 0.0429 & 0.541 & 35.6 \\
    FNO-PF
    & 0.218 & 0.208 & 0.0180 & 0.236 & 68.0 \\
    FNO-Ref.
    & 0.407 & 0.457 & 0.0368 & 0.495 & 40.4 \\
    FNO-RNO
    & 0.293 & 0.311 & 0.0249 & 0.335 & 55.0 \\
    FNO-HERO
    & \textbf{0.147}
    & \textbf{0.138}
    & \textbf{0.0168}
    & \textbf{0.153}
    & \textbf{81.7} \\
    \addlinespace[1pt]
    Trans.
    & 0.487 & 0.537 & 0.0478 & 0.618 & 31.2 \\
    Trans.-PF
    & 0.449 & 0.505 & 0.0422 & 0.562 & 35.5 \\
    Trans.-Ref.
    & 0.428 & 0.483 & 0.0401 & 0.535 & 38.7 \\
    Trans.-RNO
    & 0.384 & 0.434 & 0.0362 & 0.494 & 41.4 \\
    Trans.-HERO
    & \textbf{0.250}
    & \textbf{0.281}
    & \textbf{0.0341}
    & \textbf{0.327}
    & \textbf{49.3} \\

    \midrule
    \multicolumn{6}{c}{\textit{Kuramoto--Sivashinsky (1D)}} \\
    \midrule
    FNO
    & 0.511 & 0.578 & 0.0190 & 1.131 & 43.8 \\
    FNO-PF
    & 0.401 & 0.451 & 0.0179 & 0.972 & 44.7 \\
    FNO-Ref.
    & 0.469 & 0.528 & 0.0159 & 1.064 & 46.4 \\
    FNO-RNO
    & 0.458 & 0.515 & 0.0181 & 1.009 & 43.6 \\
    FNO-HERO
    & \textbf{0.257}
    & \textbf{0.283}
    & \textbf{0.0150}
    & \textbf{0.640}
    & \textbf{53.3} \\
    \addlinespace[1pt]
    Trans.
    & 0.611 & 0.695 & 0.0250 & 1.284 & 37.6 \\
    Trans.-PF
    & 0.501 & 0.571 & 0.0211 & 1.163 & 41.8 \\
    Trans.-Ref.
    & 0.498 & 0.572 & 0.0190 & 1.238 & 42.4 \\
    Trans.-RNO
    & 0.477 & 0.538 & 0.0201 & 1.099 & 42.6 \\
    Trans.-HERO
    & \textbf{0.321}
    & \textbf{0.361}
    & \textbf{0.0181}
    & \textbf{0.715}
    & \textbf{49.2} \\

    \midrule
    \multicolumn{6}{c}{\textit{Anisotropic Diffusion (2D)}} \\
    \midrule
    FNO
    & 0.713 & 0.825 & 0.0875 & 1.868 & 8.2 \\
    FNO-PF
    & 0.655 & 0.754 & 0.0716 & 1.355 & 12.9 \\
    FNO-Ref.
    & 0.567 & 0.652 & 0.0611 & 0.786 & 20.7 \\
    FNO-RNO
    & 0.692 & 0.778 & 0.0846 & 0.925 & 9.7 \\
    FNO-HERO
    & \textbf{0.363}
    & \textbf{0.415}
    & \textbf{0.0571}
    & \textbf{0.521}
    & \textbf{26.6} \\
    \addlinespace[1pt]
    Trans.
    & 1.168 & 1.358 & 0.0799 & 2.968 & 8.1 \\
    Trans.-PF
    & 0.965 & 1.111 & 0.0696 & 2.172 & 13.0 \\
    Trans.-Ref.
    & 0.803 & 0.943 & 0.0692 & 1.835 & 14.7 \\
    Trans.-RNO
    & 0.892 & 1.048 & 0.0751 & 2.099 & 10.4 \\
    Trans.-HERO
    & \textbf{0.537}
    & \textbf{0.626}
    & \textbf{0.0652}
    & \textbf{1.185}
    & \textbf{18.2} \\

    \midrule
    \multicolumn{6}{c}{\textit{Navier--Stokes (2D)}} \\
    \midrule
    FNO
    & 2.322 & 2.915 & 0.0449 & 13.629 & 17.6 \\
    FNO-PF
    & 1.441 & 1.682 & 0.0349 & 6.271 & 22.3 \\
    FNO-Ref.
    & 1.959 & 2.330 & 0.0402 & 9.888 & 18.1 \\
    FNO-RNO
    & 1.788 & 2.141 & 0.0388 & 7.829 & 20.7 \\
    FNO-HERO
    & \textbf{0.931}
    & \textbf{1.058}
    & 0.0325
    & \textbf{4.101}
    & \textbf{26.5} \\
    \addlinespace[1pt]
    Trans.
    & 3.075 & 3.825 & 0.0301 & 18.084 & 20.3 \\
    Trans.-PF
    & 2.400 & 2.974 & 0.0278 & 12.438 & 24.3 \\
    Trans.-Ref.
    & 2.169 & 2.591 & 0.0280 & 10.974 & 23.1 \\
    Trans.-RNO
    & 2.224 & 2.704 & 0.0291 & 11.438 & 22.5 \\
    Trans.-HERO
    & \textbf{1.423}
    & \textbf{1.687}
    & \textbf{0.0262}
    & \textbf{7.012}
    & \textbf{25.8} \\

    \midrule
    \multicolumn{6}{c}{\textit{Swift--Hohenberg (3D)}} \\
    \midrule
    FNO
    & 0.603 & 0.678 & 0.0788 & 0.941 & 13.4 \\
    FNO-PF
    & 0.514 & 0.581 & 0.0607 & 0.843 & 20.4 \\
    FNO-Ref.
    & 0.552 & 0.619 & 0.0642 & 0.871 & 19.6 \\
    FNO-RNO
    & 0.520 & 0.596 & 0.0662 & 0.850 & 19.5 \\
    FNO-HERO
    & \textbf{0.341}
    & \textbf{0.381}
    & \textbf{0.0570}
    & \textbf{0.541}
    & \textbf{27.7} \\
    \addlinespace[1pt]
    Trans.
    & 0.821 & 0.912 & 0.1079 & 1.286 & 5.3 \\
    Trans.-PF
    & 0.686 & 0.782 & 0.0914 & 1.126 & 5.9 \\
    Trans.-Ref.
    & 0.633 & 0.713 & 0.0823 & 1.004 & 11.3 \\
    Trans.-RNO
    & 0.652 & 0.739 & 0.0860 & 1.064 & 7.6 \\
    Trans.-HERO
    & \textbf{0.404}
    & \textbf{0.450}
    & \textbf{0.0780}
    & \textbf{0.667}
    & \textbf{15.0} \\

    \midrule
    \multicolumn{6}{c}{\textit{Unbalanced Advection (3D)}} \\
    \midrule
    FNO
    & 2.311 & 2.647 & 0.0869 & 7.204 & 6.4 \\
    FNO-PF
    & 1.833 & 2.121 & 0.0748 & 5.546 & 10.0 \\
    FNO-Ref.
    & 2.017 & 2.317 & 0.0776 & 6.086 & 7.8 \\
    FNO-RNO
    & 1.697 & 1.931 & 0.0791 & 4.833 & 9.0 \\
    FNO-HERO
    & \textbf{1.101}
    & \textbf{1.244}
    & 0.0702
    & 3.185
    & \textbf{12.7} \\
    \addlinespace[1pt]
    Trans.
    & 2.723 & 3.106 & 0.0731 & 7.777 & 9.8 \\
    Trans.-PF
    & 2.096 & 2.426 & 0.0736 & 6.122 & 10.3 \\
    Trans.-Ref.
    & 1.901 & 2.201 & 0.0713 & 5.319 & 9.6 \\
    Trans.-RNO
    & 1.784 & 2.079 & 0.0742 & 5.015 & 10.6 \\
    Trans.-HERO
    & \textbf{1.197}
    & \textbf{1.389}
    & \textbf{0.0673}
    & \textbf{3.175}
    & \textbf{12.1} \\
    \bottomrule
  \end{tabular*}
  \parbox{\columnwidth}{%
    \small
    Stable step denotes the mean number of stable rollout steps.
  }
\end{table}

\subsection{Ablation Study}
\label{subsec:ablation}

To examine which part of the history-enriched design produces the long-horizon gain, we conduct component-level ablations on the 2D Navier--Stokes benchmark.
All variants share the FNO backbone, data split, optimizer, evaluation protocol,
and the $K_{\mathrm{pf}}=5$ push-forward absolute objective; only two factors are
modified: whether the relative objective $\mathcal{L}_{\mathrm{rel}}$ is active,
and which branches remain in the candidate set $\mathcal{C}$ in
\eqref{eq:candidate_set}. 
\emph{\textbf{FNO}} and \emph{\textbf{FNO-PF}} are the one-step and push-forward baselines ($\lambda_{\max}=0$).
\emph{\textbf{Self-relative}} and \emph{\textbf{Lag-only HERO}} restrict
$\mathcal{C}$ to $\{\tau_{\mathrm{self}}\}$ and $\{\tau_{\mathrm{lag}}\}$, so the
reference is fixed rather than selected. \emph{\textbf{HERO w/o Lag}},
\emph{\textbf{w/o Pert}}, and \emph{\textbf{w/o Self}} each drop one branch from
the default set, whereas \emph{\textbf{HERO w/o Relative}} keeps candidate
construction and diagnostics but sets $\lambda_s=0$ in \eqref{eq:total}.

\begin{table}[t!]
  \centering
  \small
  \setlength{\tabcolsep}{3.5pt}
  \caption{Component-level ablation of HERO on the 2D Navier--Stokes
           benchmark}
  \label{tab:ablation}
  \begin{tabular}{llllll}
    \toprule
    Variant & \shortstack[l]{Val\\GM100 $\downarrow$}
            & \shortstack[l]{Test\\GM100 $\downarrow$}
            & \shortstack[l]{nRMSE\\@1 $\downarrow$}
            & \shortstack[l]{nRMSE\\@100 $\downarrow$}
            & \shortstack[l]{Stable\\step $\uparrow$} \\
    \midrule
    \midrule
    FNO               & $2.322$ & $2.915$ & $0.0449$ & $13.629$ & $17.6$ \\
    FNO-PF            & $1.441$ & $1.682$ & $0.0349$ & $6.271$  & $22.3$ \\
    Self-relative     & $1.327$ & $1.574$ & $0.0343$ & $5.714$  & $23.6$ \\
    Lag-only HERO     & $1.168$ & $1.346$ & $0.0336$ & $4.873$  & $24.8$ \\
    HERO w/o Lag      & $1.205$ & $1.397$ & $0.0335$ & $5.026$  & $24.5$ \\
    HERO w/o Pert     & $1.061$ & $1.221$ & $0.0329$ & $4.527$  & $25.4$ \\
    HERO w/o Self     & $1.024$ & $1.178$ & $0.0331$ & $4.381$  & $25.7$ \\
    HERO w/o Relative & $1.425$ & $1.664$ & $0.0348$ & $6.198$  & $22.5$ \\
    Full HERO         & $\mathbf{0.931}$ & $\mathbf{1.058}$ & $\mathbf{0.0325}$
                      & $\mathbf{4.101}$ & $\mathbf{26.5}$ \\
    \bottomrule
  \end{tabular}
\end{table}

Table~\ref{tab:ablation} reports the ablation results. \textbf{HERO w/o
Relative} performs at the level of \textbf{FNO-PF} (Test GM100: $1.664$ vs.\ $1.682$), confirming that candidate construction and diagnostics contribute nothing by themselves and that the gain comes from the relative gradient. 
Activating it reduces Test GM100 from $1.664$ to $1.058$ and
nRMSE@100 from $6.198$ to $4.101$, approximately $36.4\%$ and $33.8\%$, while nRMSE@1 changes only from $0.0348$ to $0.0325$, showing that relative supervision acts on error accumulation rather than on one-step fitting.
\textbf{Self-relative} improves over \textbf{FNO-PF} but remains far from \textbf{Full HERO} ($1.574$ vs.\ $1.058$),
since a reference drawn from the current parameters provides little contrast, and \textbf{Lag-only HERO} recovers only part of the gap ($1.346$) with a fixed-lag reference. 
Among the reduced candidate sets, \textbf{HERO w/o Lag} degrades most ($1.397$), whereas \textbf{w/o Pert} ($1.221$) and \textbf{w/o Self} ($1.178$) lose less; although \textbf{w/o Self} comes closest overall, its shorter stable horizon ($25.7$ vs.\ $26.5$) localizes the residual difference in late-rollout
stability. 
Overall, these results show that the long-horizon gain of HERO arises from relative supervision against a dynamically history-enriched failure reference, with the lagged branch and diagnostic-based selection both necessary for the full effect.

\subsection{Out-of-Distribution Generalization}
\label{subsec:ood}

To test whether HERO remains stable outside the training regime, we conduct zero-shot evaluation on the 2D Kolmogorov flow. All models are trained at $\nu=1.00\times10^{-2}$, $A_f=1.0$, and $N=64^{2}$, then evaluated without fine-tuning under i) parameter shifts to $\nu=6.67\times10^{-3}$ or $A_f=1.50$ and ii) resolution shifts to $96^{2}$ or $128^{2}$. All methods use the same initial conditions and rollout protocol. We report GM$_{100}$, nRMSE@200, and the mean stability horizon (MSH).
Table~\ref{tab:ood} reports that HERO achieves the lowest error and longest stability horizon under all four shifts. Under parameter extrapolation, HERO reduces GM$_{100}$ from $1.487$ to $0.742$ for lower viscosity and from $1.603$ to $0.809$ for stronger forcing relative to FNO. Compared with FNO-PF, the strongest baseline, HERO reduces GM$_{100}$ by $37.9\%$ and $38.6\%$, respectively, while approximately doubling MSH.
HERO also maintains its advantage under resolution extrapolation. At $96^{2}$ and $128^{2}$, it reduces GM$_{100}$ from $1.274$ to $0.654$ and from $1.361$ to $0.704$ relative to FNO, respectively. Compared with FNO-PF, these values correspond to reductions of $37.2\%$ and $37.5\%$, while MSH increases from $1.9$ to $3.5$ and from $1.6$ to $3.0$. These results indicate that HERO suppresses long-horizon error accumulation under unseen physical parameters and spatial resolutions.

\begin{table}[t!]
  \centering
  \caption{Zero-shot out-of-distribution generalization on the\\
           2D Kolmogorov flow under parameter and resolution shift}
  \label{tab:ood}
  \begin{threeparttable}
  \small
  \renewcommand{\arraystretch}{0.85}
    \begin{tabular}{llll}
      \toprule
      Method & GM$_{100}$ $\downarrow$ & nRMSE@200 $\downarrow$ & MSH $\uparrow$ \\
      \midrule
      \midrule
      \multicolumn{4}{l}{\emph{Parameter}: $\nu=6.67\times10^{-3}$, $A_f=1.0$, $N=64^{2}$} \\
      FNO      & $1.487$ & $2.286$ & $1.2$ \\
      FNO-PF   & $1.194$ & $1.936$ & $1.7$ \\
      FNO-RNO  & $1.266$ & $2.044$ & $1.5$ \\
      FNO-HERO & $\mathbf{0.742}$ & $\mathbf{1.211}$ & $\mathbf{3.1}$ \\
      \midrule
      \multicolumn{4}{l}{\emph{Parameter}: $\nu=1.00\times10^{-2}$, $A_f=1.50$, $N=64^{2}$} \\
      FNO      & $1.603$ & $2.472$ & $1.0$ \\
      FNO-PF   & $1.318$ & $2.127$ & $1.4$ \\
      FNO-RNO  & $1.374$ & $2.216$ & $1.3$ \\
      FNO-HERO & $\mathbf{0.809}$ & $\mathbf{1.336}$ & $\mathbf{2.7}$ \\
      \midrule
      \multicolumn{4}{l}{\emph{Resolution}: $\nu=1.00\times10^{-2}$, $A_f=1.0$, $N=96^{2}$} \\
      FNO      & $1.274$ & $1.754$ & $1.3$ \\
      FNO-PF   & $1.041$ & $1.493$ & $1.9$ \\
      FNO-RNO  & $1.096$ & $1.572$ & $1.7$ \\
      FNO-HERO & $\mathbf{0.654}$ & $\mathbf{0.946}$ & $\mathbf{3.5}$ \\
      \midrule
      \multicolumn{4}{l}{\emph{Resolution}: $\nu=1.00\times10^{-2}$, $A_f=1.0$, $N=128^{2}$} \\
      FNO      & $1.361$ & $1.967$ & $1.1$ \\
      FNO-PF   & $1.126$ & $1.681$ & $1.6$ \\
      FNO-RNO  & $1.187$ & $1.774$ & $1.5$ \\
      FNO-HERO & $\mathbf{0.704}$ & $\mathbf{1.058}$ & $\mathbf{3.0}$ \\
      \bottomrule
    \end{tabular}
  \end{threeparttable}
\end{table}

\section{Conclusion}

HERO introduced history-enriched rollout training for long-horizon autoregressive neural operators. Instead of relying only on absolute trajectory error, HERO compares the current rollout with dynamically selected failure trajectories from earlier optimization states and converts this comparison into a margin-based relative objective. Experiments across nine PDE benchmarks with spectral and attention-based backbones showed consistent improvements in long-horizon accuracy, stable rollout length, and out-of-distribution robustness, without changing the inference-time architecture or cost. Ablation studies further confirmed that the gains arise from relative supervision and that the lagged branch and diagnostic-based reference selection are necessary for the full improvement. These results indicate that optimization history provides an effective training signal for suppressing autoregressive error accumulation in neural PDE surrogates.

\section{Limitations and Ethical Considerations}
This work uses only simulated PDE data and does not involve human subjects, personal information, or privacy-sensitive content.

The current evaluation focuses on the PDE benchmarks and backbone architectures. HERO also introduces additional training-time computation for candidate rollout construction and reference selection, while leaving the inference-time architecture and cost unchanged.

\section{Generative AI Usage}
Generative AI was used only for language editing.

\bibliographystyle{ACM-Reference-Format}
\bibliography{software}

@String{Computing = "Computing" }

@String{Computer = "{IEEE} Computer" }

@BOOK{test,
   author = "Donald E. Knuth",
   title = "Seminumerical Algorithms",
   volume = 2,
   series = "The Art of Computer Programming",
   publisher = "Addison-Wesley",
   address = "Reading, MA",
   edition = "2nd",
   month = "10~" # jan,
   year = "1981",
}

@ArtifactSoftware{R,
    title = {R: A Language and Environment for Statistical Computing},
    author = {{R Core Team}},
    organization = {R Foundation for Statistical Computing},
    address = {Vienna, Austria},
    year = {2019},
    url = {https://www.R-project.org/},
}

@article{kovachki2023neural,
  title={Neural operator: Learning maps between function spaces with applications to pdes},
  author={Kovachki, Nikola and Li, Zongyi and Liu, Burigede and Azizzadenesheli, Kamyar and Bhattacharya, Kaushik and Stuart, Andrew and Anandkumar, Anima},
  journal={Journal of Machine Learning Research},
  volume={24},
  number={89},
  pages={1--97},
  year={2023}
}

@inproceedings{li2020fourier,
  title={Fourier neural operator for parametric partial differential equations},
  author={Li, Zongyi and Kovachki, Nikola and Azizzadenesheli, Kamyar and Liu, Burigede and Bhattacharya, Kaushik and Stuart, Andrew and Anandkumar, Anima},
  booktitle={International Conference on Learning Representations},
  year={2021},
  url={https://openreview.net/forum?id=c8P9NQVtmnO}
}

@inproceedings{BrandstetterWW22,
  author       = {Johannes Brandstetter and
                  Daniel E. Worrall and
                  Max Welling},
  title        = {Message Passing Neural {PDE} Solvers},
  booktitle    = {The Tenth International Conference on Learning Representations, {ICLR}
                  2022, Virtual Event, April 25-29, 2022},
  year         = {2022},
}

@inproceedings{um2020solver,
  title={Solver-in-the-loop: Learning from differentiable physics to interact with iterative pde-solvers},
  author={Um, Kiwon and Brand, Robert and Fei, Yun Raymond and Holl, Philipp and Thuerey, Nils},
  booktitle={Advances in Neural Information Processing Systems},
  volume={33},
  pages={6111--6122},
  year={2020}
}

@inproceedings{sanchez2020learning,
  title={Learning to simulate complex physics with graph networks},
  author={Sanchez-Gonzalez, Alvaro and Godwin, Jonathan and Pfaff, Tobias and Ying, Rex and Leskovec, Jure and Battaglia, Peter},
  booktitle={International conference on machine learning},
  pages={8459--8468},
  year={2020},
  organization={PMLR}
}

@inproceedings{PfaffFSB21,
  author       = {Tobias Pfaff and
                  Meire Fortunato and
                  Alvaro Sanchez{-}Gonzalez and
                  Peter W. Battaglia},
  title        = {Learning Mesh-Based Simulation with Graph Networks},
  booktitle    = {9th International Conference on Learning Representations, {ICLR} 2021,
                  Virtual Event, Austria, May 3-7, 2021},
  year         = {2021},
}

@inproceedings{koehler2024apebench,
  title={APEBench: A benchmark for autoregressive neural emulators of PDEs},
  author={Koehler, Felix and Niedermayr, Simon and Westermann, R{\"u}diger and Thuerey, Nils},
  booktitle={Advances in Neural Information Processing Systems},
  volume={37},
  pages={120252--120310},
  year={2024}
}

@article{list2025differentiability,
  title={Differentiability in unrolled training of neural physics simulators on transient dynamics},
  author={List, Bjoern and Chen, Li-Wei and Bali, Kartik and Thuerey, Nils},
  journal={Computer Methods in Applied Mechanics and Engineering},
  volume={433},
  pages={117441},
  year={2025},
  publisher={Elsevier}
}

@article{McCabeHSB23,
  author       = {Michael McCabe and
                  Peter Harrington and
                  Shashank Subramanian and
                  Jed Brown},
  title        = {Towards Stability of Autoregressive Neural Operators},
  journal      = {Trans. Mach. Learn. Res.},
  volume       = {2023},
  year         = {2023},
  url          = {https://openreview.net/forum?id=RFfUUtKYOG},
  bibsource    = {dblp computer science bibliography, https://dblp.org}
}

@inproceedings{lippe2023pde,
  title={Pde-refiner: Achieving accurate long rollouts with neural pde solvers},
  author={Lippe, Phillip and Veeling, Bas and Perdikaris, Paris and Turner, Richard and Brandstetter, Johannes},
  booktitle={Advances in Neural Information Processing Systems},
  volume={36},
  pages={67398--67433},
  year={2023}
}

@article{li2026sgno,
  title={SGNO: Spectral Generator Neural Operators for Stable Long Horizon PDE Rollouts},
  author={Li, Jiayi and Jiang, Penghao and Saleem, Hira and Wang, Zhaonan and Koniusz, Piotr and Salim, Flora D},
  journal={arXiv preprint arXiv:2602.18801},
  year={2026}
}

@inproceedings{jiang2026hierarchical,
  title={Hierarchical implicit neural emulators},
  author={Jiang, Ruoxi and Zhang, Xiao and Jakhar, Karan and Lu, Peter Y and Hassanzadeh, Pedram and Maire, Michael and Willett, Rebecca},
  booktitle={Advances in Neural Information Processing Systems},
  volume={38},
  pages={73718--73751},
  year={2025}
}

@inproceedings{HuangG25,
  author       = {Yunfei Huang and
                  David S. Greenberg},
  title        = {Geometric and Physical Constraints Synergistically Enhance Neural
                  {PDE} Surrogates},
  booktitle    = {Forty-second International Conference on Machine Learning, {ICML}
                  2025, Vancouver, BC, Canada, July 13-19, 2025},
  volume       = {267},
  year         = {2025},
}

@inproceedings{li2022learning,
  title={Learning chaotic dynamics in dissipative systems},
  author={Li, Zongyi and Liu-Schiaffini, Miguel and Kovachki, Nikola and Azizzadenesheli, Kamyar and Liu, Burigede and Bhattacharya, Kaushik and Stuart, Andrew and Anandkumar, Anima},
  booktitle={Advances in Neural Information Processing Systems},
  volume={35},
  pages={16768--16781},
  year={2022}
}

@inproceedings{jiang2023training,
  title={Training neural operators to preserve invariant measures of chaotic attractors},
  author={Jiang, Ruoxi and Lu, Peter Y and Orlova, Elena and Willett, Rebecca},
  booktitle={Advances in Neural Information Processing Systems},
  volume={36},
  pages={27645--27669},
  year={2023}
}

@inproceedings{SchiffWPHKSZ24,
  author       = {Yair Schiff and
                  Zhong Yi Wan and
                  Jeffrey B. Parker and
                  Stephan Hoyer and
                  Volodymyr Kuleshov and
                  Fei Sha and
                  Leonardo Zepeda{-}N{\'{u}}{\~{n}}ez},
  title        = {DySLIM: Dynamics Stable Learning by Invariant Measure for Chaotic
                  Systems},
  booktitle    = {Forty-first International Conference on Machine Learning, {ICML} 2024,
                  Vienna, Austria, July 21-27, 2024},
  volume       = {235},
  pages        = {43649--43684},
  year         = {2024},
}

@article{ye2025recurrent,
  title={Recurrent Neural Operators: Stable Long-Term PDE Prediction},
  author={Ye, Zaijun and Zhang, Chen-Song and Wang, Wansheng},
  journal={arXiv preprint arXiv:2505.20721},
  year={2025}
}

@inproceedings{li2023scalable,
  title={Scalable transformer for pde surrogate modeling},
  author={Li, Zijie and Shu, Dule and Barati Farimani, Amir},
  booktitle={Advances in Neural Information Processing Systems},
  volume={36},
  pages={28010--28039},
  year={2023}
}

@inproceedings{wu2024Transolver,
  title={Transolver: A Fast Transformer Solver for PDEs on General Geometries},
  author={Haixu Wu and Huakun Luo and Haowen Wang and Jianmin Wang and Mingsheng Long},
  booktitle={International Conference on Machine Learning},
  year={2024}
}

@inproceedings{gupta2021multiwaveletbasedoperatorlearningdifferential,
  title={Multiwavelet-based Operator Learning for Differential Equations},
  author={Gaurav Gupta and Xiongye Xiao and Paul Bogdan},
  booktitle={Advances in Neural Information Processing Systems},
  volume={34},
  pages={24048--24062},
  year={2021},
  url={https://proceedings.neurips.cc/paper/2021/hash/c9e5c2b59d98488fe1070e744041ea0e-Abstract.html}
}

@inproceedings{DBLP:journals/corr/abs-2105-14995,
  author={Shuhao Cao},
  title={Choose a Transformer: Fourier or Galerkin},
  booktitle={Advances in Neural Information Processing Systems},
  volume={34},
  year={2021}
}

@article{JMLR:v24:23-0064,
  author  = {Zongyi Li and Daniel Zhengyu Huang and Burigede Liu and Anima Anandkumar},
  title   = {Fourier Neural Operator with Learned Deformations for PDEs on General Geometries},
  journal = {Journal of Machine Learning Research},
  year    = {2023},
  volume  = {24},
  number  = {388},
  pages   = {1--26},
  url     = {http://jmlr.org/papers/v24/23-0064.html}
}

@InProceedings{pmlr-v202-hao23c,
  title = 	 {{GNOT}: A General Neural Operator Transformer for Operator Learning},
  author =       {Hao, Zhongkai and Wang, Zhengyi and Su, Hang and Ying, Chengyang and Dong, Yinpeng and Liu, Songming and Cheng, Ze and Song, Jian and Zhu, Jun},
  booktitle = 	 {Proceedings of the 40th International Conference on Machine Learning},
  pages = 	 {12556--12569},
  year = 	 {2023},
  editor = 	 {Krause, Andreas and Brunskill, Emma and Cho, Kyunghyun and Engelhardt, Barbara and Sabato, Sivan and Scarlett, Jonathan},
  volume = 	 {202},
  series = 	 {Proceedings of Machine Learning Research},
  month = 	 {23--29 Jul},
  publisher =    {PMLR},
  url = 	 {https://proceedings.mlr.press/v202/hao23c.html}
}

@inproceedings{tran2023factorized,
  title     = {Factorized Fourier Neural Operators},
  author    = {Alasdair Tran and Alexander Mathews and Lexing Xie and Cheng Soon Ong},
  booktitle = {The Eleventh International Conference on Learning Representations},
  year      = {2023},
  url       = {https://openreview.net/forum?id=tmIiMPl4IPa}
}

@inproceedings{raonic2023convolutionalneuraloperatorsrobust,
  title={Convolutional Neural Operators for robust and accurate learning of PDEs},
  author={Bogdan Raonić and Roberto Molinaro and Tim De Ryck and Tobias Rohner and Francesca Bartolucci and Rima Alaifari and Siddhartha Mishra and Emmanuel de Bézenac},
  booktitle={Advances in Neural Information Processing Systems},
  volume={36},
  year={2023}
}

@InProceedings{pmlr-v202-bonev23a,
  title = 	 {Spherical {F}ourier Neural Operators: Learning Stable Dynamics on the Sphere},
  author =       {Bonev, Boris and Kurth, Thorsten and Hundt, Christian and Pathak, Jaideep and Baust, Maximilian and Kashinath, Karthik and Anandkumar, Anima},
  booktitle = 	 {Proceedings of the 40th International Conference on Machine Learning},
  pages = 	 {2806--2823},
  year = 	 {2023},
  editor = 	 {Krause, Andreas and Brunskill, Emma and Cho, Kyunghyun and Engelhardt, Barbara and Sabato, Sivan and Scarlett, Jonathan},
  volume = 	 {202},
  series = 	 {Proceedings of Machine Learning Research},
  month = 	 {23--29 Jul},
  publisher =    {PMLR},
  url = 	 {https://proceedings.mlr.press/v202/bonev23a.html}
}

@article{Cao2024Laplace,
  author = {Cao, Qianying and Goswami, Somdatta and Karniadakis, George Em},
  title = {Laplace neural operator for solving differential equations},
  journal = {Nature Machine Intelligence},
  year = {2024},
  volume = {6},
  number = {6},
  pages = {631--640},
  doi = {10.1038/s42256-024-00844-4},
  url = {https://doi.org/10.1038/s42256-024-00844-4},
  isbn = {2522-5839}
}

@article{10.1145/3648506,
author = {Li, Zongyi and Zheng, Hongkai and Kovachki, Nikola and Jin, David and Chen, Haoxuan and Liu, Burigede and Azizzadenesheli, Kamyar and Anandkumar, Anima},
title = {Physics-Informed Neural Operator for Learning Partial Differential Equations},
journal = {ACM/IMS Journal of Data Science},
year = {2024},
issue_date = {September 2024},
publisher = {Association for Computing Machinery},
address = {New York, NY, USA},
volume = {1},
number = {3},
url = {https://doi.org/10.1145/3648506},
doi = {10.1145/3648506},
month = may,
articleno = {9},
numpages = {27}
}

@InProceedings{pmlr-v235-hao24d,
  title = 	 {{DPOT}: Auto-Regressive Denoising Operator Transformer for Large-Scale {PDE} Pre-Training},
  author =       {Hao, Zhongkai and Su, Chang and Liu, Songming and Berner, Julius and Ying, Chengyang and Su, Hang and Anandkumar, Anima and Song, Jian and Zhu, Jun},
  booktitle = 	 {Proceedings of the 41st International Conference on Machine Learning},
  pages = 	 {17616--17635},
  year = 	 {2024},
  editor = 	 {Salakhutdinov, Ruslan and Kolter, Zico and Heller, Katherine and Weller, Adrian and Oliver, Nuria and Scarlett, Jonathan and Berkenkamp, Felix},
  volume = 	 {235},
  series = 	 {Proceedings of Machine Learning Research},
  month = 	 {21--27 Jul},
  publisher =    {PMLR},
  url = 	 {https://proceedings.mlr.press/v235/hao24d.html}
}

@inproceedings{herde2024poseidonefficientfoundationmodels,
  title={Poseidon: Efficient Foundation Models for PDEs},
  author={Maximilian Herde and Bogdan Raonić and Tobias Rohner and Roger Käppeli and Roberto Molinaro and Emmanuel de Bézenac and Siddhartha Mishra},
  booktitle={Advances in Neural Information Processing Systems},
  volume={37},
  year={2024}
}

@inproceedings{chen2025dataefficientoperatorlearningunsupervised,
  title={Data-Efficient Operator Learning via Unsupervised Pretraining and In-Context Learning},
  author={Wuyang Chen and Jialin Song and Pu Ren and Shashank Subramanian and Dmitriy Morozov and Michael W. Mahoney},
  booktitle={Advances in Neural Information Processing Systems},
  volume={37},
  year={2024}
}

@inproceedings{takamoto2024pdebenchextensivebenchmarkscientific,
  title={PDEBENCH: An Extensive Benchmark for Scientific Machine Learning},
  author={Makoto Takamoto and Timothy Praditia and Raphael Leiteritz and Dan MacKinlay and Francesco Alesiani and Dirk Pflüger and Mathias Niepert},
  booktitle={Advances in Neural Information Processing Systems},
  volume={35},
  year={2022}
}

@InProceedings{pmlr-v15-ross11a,
  title = 	 {A Reduction of Imitation Learning and Structured Prediction to No-Regret Online Learning},
  author = 	 {Ross, Stephane and Gordon, Geoffrey and Bagnell, Drew},
  booktitle = 	 {Proceedings of the Fourteenth International Conference on Artificial Intelligence and Statistics},
  pages = 	 {627--635},
  year = 	 {2011},
  editor = 	 {Gordon, Geoffrey and Dunson, David and Dudík, Miroslav},
  volume = 	 {15},
  series = 	 {Proceedings of Machine Learning Research},
  address = 	 {Fort Lauderdale, FL, USA},
  month = 	 {11--13 Apr},
  publisher =    {PMLR},
  url = 	 {https://proceedings.mlr.press/v15/ross11a.html}
}

@inproceedings{DBLP:journals/corr/BengioVJS15,
  author={Samy Bengio and Oriol Vinyals and Navdeep Jaitly and Noam Shazeer},
  title={Scheduled Sampling for Sequence Prediction with Recurrent Neural Networks},
  booktitle={Advances in Neural Information Processing Systems},
  volume={28},
  year={2015}
}

@inproceedings{lamb2016professorforcingnewalgorithm,
  title={Professor Forcing: A New Algorithm for Training Recurrent Networks},
  author={Alex Lamb and Anirudh Goyal and Ying Zhang and Saizheng Zhang and Aaron Courville and Yoshua Bengio},
  booktitle={Advances in Neural Information Processing Systems},
  volume={29},
  year={2016}
}

@inproceedings{10.1145/1553374.1553380,
author = {Bengio, Yoshua and Louradour, J\'{e}r\^{o}me and Collobert, Ronan and Weston, Jason},
title = {Curriculum learning},
booktitle = {Proceedings of the 26th Annual International Conference on Machine Learning},
year = {2009},
isbn = {9781605585161},
publisher = {Association for Computing Machinery},
address = {New York, NY, USA},
url = {https://doi.org/10.1145/1553374.1553380},
doi = {10.1145/1553374.1553380},
pages = {41--48},
numpages = {8},
location = {Montreal, Quebec, Canada},
series = {ICML '09}
}

@InProceedings{pmlr-v80-ren18a,
  title = 	 {Learning to Reweight Examples for Robust Deep Learning},
  author =       {Ren, Mengye and Zeng, Wenyuan and Yang, Bin and Urtasun, Raquel},
  booktitle = 	 {Proceedings of the 35th International Conference on Machine Learning},
  pages = 	 {4334--4343},
  year = 	 {2018},
  editor = 	 {Dy, Jennifer and Krause, Andreas},
  volume = 	 {80},
  series = 	 {Proceedings of Machine Learning Research},
  month = 	 {10--15 Jul},
  publisher =    {PMLR},
  url = 	 {https://proceedings.mlr.press/v80/ren18a.html}
}

@article{RAISSI2019686,
title = {Physics-informed neural networks: A deep learning framework for solving forward and inverse problems involving nonlinear partial differential equations},
journal = {Journal of Computational Physics},
volume = {378},
pages = {686-707},
year = {2019},
issn = {0021-9991},
doi = {https://doi.org/10.1016/j.jcp.2018.10.045},
url = {https://www.sciencedirect.com/science/article/pii/S0021999118307125},
author = {M. Raissi and P. Perdikaris and G.E. Karniadakis}
}

@inproceedings{zhong2019symplectic,
  title={Symplectic ode-net: Learning hamiltonian dynamics with control},
  author={Zhong, Yaofeng Desmond and Dey, Biswadip and Chakraborty, Amit},
  booktitle={International Conference on Learning Representations},
  year={2020},
  url={https://openreview.net/forum?id=ryxmb1rKDS}
}

@article{karniadakis2021physics,
  title={Physics-informed machine learning},
  author={Karniadakis, George Em and Kevrekidis, Ioannis G. and Lu, Lu and Perdikaris, Paris and Wang, Sifan and Yang, Liu},
  journal={Nature Reviews Physics},
  volume={3},
  pages={422--440},
  year={2021},
  doi={10.1038/s42254-021-00314-5},
  url={https://doi.org/10.1038/s42254-021-00314-5}
}

@article{brunton2020machine,
  title={Machine learning for fluid mechanics},
  author={Brunton, Steven L and Noack, Bernd R and Koumoutsakos, Petros},
  journal={Annual review of fluid mechanics},
  volume={52},
  number={1},
  pages={477--508},
  year={2020},
  publisher={Annual Reviews}
}

@article{kochkov2021machine,
  title={Machine learning--accelerated computational fluid dynamics},
  author={Kochkov, Dmitrii and Smith, Jamie A. and Alieva, Ayya and Wang, Qing and Brenner, Michael P. and Hoyer, Stephan},
  journal={Proceedings of the National Academy of Sciences},
  volume={118},
  number={21},
  pages={e2101784118},
  year={2021},
  doi={10.1073/pnas.2101784118},
  url={https://doi.org/10.1073/pnas.2101784118}
}

@article{rasp2020weatherbench,
  title={{WeatherBench}: A Benchmark Data Set for Data-Driven Weather Forecasting},
  author={Rasp, Stephan and Dueben, Peter D. and Scher, Sebastian and Weyn, Jonathan A. and Mouatadid, Said and Thuerey, Nils},
  journal={Journal of Advances in Modeling Earth Systems},
  volume={12},
  pages={e2020MS002203},
  year={2020},
  doi={10.1029/2020MS002203},
  url={https://doi.org/10.1029/2020MS002203}
}

@incollection{BURGERS1948171,
title = {A Mathematical Model Illustrating the Theory of Turbulence},
editor = {Richard {Von Mises} and Theodore {Von Kármán}},
series = {Advances in Applied Mechanics},
publisher = {Elsevier},
volume = {1},
pages = {171-199},
year = {1948},
issn = {0065-2156},
doi = {https://doi.org/10.1016/S0065-2156(08)70100-5},
url = {https://www.sciencedirect.com/science/article/pii/S0065215608701005},
author = {J.M. Burgers}
}

@article{Korteweg01051895,
author = {D. J. Korteweg and G. de Vries},
title = {XLI. On the change of form of long waves advancing in a rectangular canal, and on a new type of long stationary waves },
journal = {The London, Edinburgh, and Dublin Philosophical Magazine and Journal of Science},
volume = {39},
number = {240},
pages = {422--443},
year = {1895},
publisher = {Taylor \& Francis},
doi = {10.1080/14786449508620739},
URL = {https://doi.org/10.1080/14786449508620739},
eprint = {https://doi.org/10.1080/14786449508620739}
}

@article{SIVASHINSKY19771177,
title = {Nonlinear analysis of hydrodynamic instability in laminar flames—I. Derivation of basic equations},
journal = {Acta Astronautica},
volume = {4},
number = {11},
pages = {1177-1206},
year = {1977},
issn = {0094-5765},
doi = {https://doi.org/10.1016/0094-5765(77)90096-0},
url = {https://www.sciencedirect.com/science/article/pii/0094576577900960},
author = {G.I. Sivashinsky}
}

@article{10.1143/PTPS.64.346,
    author = {Kuramoto, Yoshiki},
    title = {Diffusion-Induced Chaos in Reaction Systems},
    journal = {Progress of Theoretical Physics Supplement},
    volume = {64},
    pages = {346-367},
    year = {1978},
    month = {02},
    issn = {0375-9687},
    doi = {10.1143/PTPS.64.346},
    url = {https://doi.org/10.1143/PTPS.64.346},
    eprint = {https://academic.oup.com/ptps/article-pdf/doi/10.1143/PTPS.64.346/5293041/64-346.pdf},
}

@article{fick1855ueber,
  title={Ueber diffusion},
  author={Fick, Adolf},
  journal={Annalen der physik},
  volume={170},
  number={1},
  pages={59--86},
  year={1855},
  publisher={Wiley Online Library}
}

@article{MESHALKIN19611700,
title = {Investigation of the stability of a stationary solution of a system of equations for the plane movement of an incompressible viscous liquid},
journal = {Journal of Applied Mathematics and Mechanics},
volume = {25},
number = {6},
pages = {1700-1705},
year = {1961},
issn = {0021-8928},
doi = {https://doi.org/10.1016/0021-8928(62)90149-1},
url = {https://www.sciencedirect.com/science/article/pii/0021892862901491},
author = {L.D. Meshalkin and Ia.G. Sinai}
}

@article{navier1823memoire,
  title={M{\'e}moire sur les lois du mouvement des fluides},
  author={Navier, CLMH and others},
  journal={M{\'e}moires de l’Acad{\'e}mie Royale des Sciences de l’Institut de France},
  volume={6},
  number={1823},
  pages={389--440},
  year={1823}
}

@article{PhysRevA.15.319,
  title = {Hydrodynamic fluctuations at the convective instability},
  author = {Swift, J. and Hohenberg, P. C.},
  journal = {Phys. Rev. A},
  volume = {15},
  issue = {1},
  pages = {319--328},
  numpages = {0},
  year = {1977},
  month = {Jan},
  publisher = {American Physical Society},
  doi = {10.1103/PhysRevA.15.319},
  url = {https://link.aps.org/doi/10.1103/PhysRevA.15.319}
}

@article{Courant1928Ber,
  author = {Courant, R. and Friedrichs, K. and Lewy, H.},
  title = {Über die partiellen Differenzengleichungen der mathematischen Physik},
  journal = {Mathematische Annalen},
  year = {1928},
  volume = {100},
  number = {1},
  pages = {32--74},
  doi = {10.1007/BF01448839},
  url = {https://doi.org/10.1007/BF01448839},
  isbn = {1432-1807}
}

\appendix

\section{Related Works}
\label{app:rl}
\subsection{Rollout Training with Self-Generated States}

Neural operators learn mappings between function spaces and provide efficient surrogate models for parameterized PDEs~\cite{kovachki2023neural}. The FNO parameterizes global integral kernels in the Fourier domain and is widely used as a backbone for time-dependent PDE modeling~\cite{li2020fourier}.
Other operator families use multiwavelet bases, Fourier- or Galerkin-style
attention, and factorized Fourier layers~\cite{gupta2021multiwaveletbasedoperatorlearningdifferential,DBLP:journals/corr/abs-2105-14995,tran2023factorized}.
Learned spatial deformations, spherical representations, and Laplace-domain
modeling adapt operators to different geometries or dynamics~\cite{JMLR:v24:23-0064,pmlr-v202-bonev23a,Cao2024Laplace}.
Continuous convolutions and attention-based neural operators provide further architectural
alternatives~\cite{raonic2023convolutionalneuraloperatorsrobust,pmlr-v202-hao23c}.
Physics-informed methods encode governing-equation constraints in neural
surrogates~\cite{RAISSI2019686,10.1145/3648506,karniadakis2021physics}.
Pretraining and foundation-model approaches instead emphasize transfer across
PDEs and data efficiency~\cite{pmlr-v235-hao24d,herde2024poseidonefficientfoundationmodels,chen2025dataefficientoperatorlearningunsupervised}.
When a local evolution operator is recursively applied to its own predictions, local errors enter subsequent inputs and accumulate throughout the rollout.
Rollout-based training reduces this train--test mismatch by exposing the
model to self-generated states. Message Passing Neural PDE Solvers combine temporal bundling with push-forward training~\cite{BrandstetterWW22},
whereas Solver-in-the-Loop couples learned corrections with differentiable numerical solvers during multi-step optimization~\cite{um2020solver}.
Graph-based simulators similarly improve autoregressive robustness through
recursive training and perturbations of training states
~\cite{sanchez2020learning,PfaffFSB21}. APEBench organizes related
strategies into one-step learning, supervised unrolling, diverted-chain
training, and differentiable solver interaction
~\cite{koehler2024apebench}. List et al.\ distinguish the effect
of exposure to model-induced states from that of long-term gradients
~\cite{list2025differentiability}, while recurrent neural operators
directly train on recursively generated predictions
~\cite{ye2025recurrent}.
More generally, scheduled sampling and Professor Forcing reduce the
teacher-forcing/free-running mismatch in sequence modeling, while DAgger
aggregates on-policy data to address covariate shift in imitation
learning~\cite{DBLP:journals/corr/BengioVJS15,lamb2016professorforcingnewalgorithm,pmlr-v15-ross11a}.
Curriculum learning and example reweighting also adapt training signals, but do
not construct a relative target from an explicitly selected historical
rollout~\cite{10.1145/1553374.1553380,pmlr-v80-ren18a}.
These methods align training more closely with autoregressive inference,
but their supervision predominantly measures the absolute discrepancy between the current rollout and the ground-truth trajectory. 
Such objectives penalize the current prediction error without comparing it against long-horizon failure behaviors produced earlier during optimization. HERO complements this absolute supervision with a relative
objective constructed from dynamically selected failure trajectories in the model's optimization history.

\subsection{Long-Horizon Structural and Statistical Stability}

A complementary line of research controls long-horizon degradation through
spectral modeling, architectural constraints, and structure-preserving
objectives. McCabe et al.\ connect autoregressive instability to aliasing
and spectral sensitivity~\cite{McCabeHSB23}, while PDE-Refiner improves
the representation of low-amplitude frequency components through iterative
refinement~\cite{lippe2023pde}. The Spectral Generator Neural Operator constrains repeatedly propagated spectral dynamics
~\cite{li2026sgno}, and hierarchical implicit neural emulators introduce multiscale representations of future states to improve temporal coherence
~\cite{jiang2026hierarchical}. Geometric and physical constraints further encode symmetry and conservation structure into neural PDE surrogates
~\cite{HuangG25}.
Physics-informed formulations impose governing-equation residuals, whereas
structure-preserving approaches can encode geometric invariants in learned dynamical
systems~\cite{RAISSI2019686,karniadakis2021physics,zhong2019symplectic}.
For dissipative chaotic systems, where long-term pointwise tracking is limited by sensitivity to initial conditions, prior work instead targets
dissipative behavior and long-term statistics. Existing methods constrain learned dynamics on the attractor~\cite{li2022learning}, preserve invariant
measures through optimal-transport and contrastive objectives
~\cite{jiang2023training}, or introduce scalable measure-matching regularization~\cite{SchiffWPHKSZ24}.
At the broader application scale, fluid simulation and data-driven weather
forecasting motivate evaluating multi-step trajectories and forecast lead times
alongside one-step error~\cite{brunton2020machine,kochkov2021machine,rasp2020weatherbench}.
These approaches improve long-horizon behavior through operator design, iterative refinement, auxiliary representations, or predefined physical
and statistical constraints. HERO addresses a complementary supervision
problem: it leaves the inference-time operator unchanged and uses training-only references from the model's optimization history to measure whether the current rollout has improved over previously observed failure
behavior.

\section{Method}
\subsection{Gradient Flow Analysis}
\label{subsec:gradient}

HERO separates reference construction from parameter optimization. Within
the relative branch, the current rollout $\tau_{\theta_s}$ is the only
differentiable trajectory. All active candidate trajectories, together
with the diagnostic values, candidate scores $S^{(j)}$, selected failure
reference $\tau^{-}$, and reference error $d_{\mathrm{fail}}$, are
detached from the computation graph. Consequently, gradients do not
propagate through the lagged operator, perturbed branch, diagnostic
computation, or candidate-selection procedure.

For the sample-level relative loss, the detached reference satisfies
$\nabla_{\theta_s}d_{\mathrm{fail}}=0$. Its gradient is therefore
\begin{equation}
\nabla_{\theta_s}\ell_{\mathrm{rel}}
=
\sigma
\left(
\beta
\left[
d_{\mathrm{cur}}
-
d_{\mathrm{fail}}
+
m
\right]
\right)
\nabla_{\theta_s}d_{\mathrm{cur}},
\end{equation}
where $\sigma(x)=1/(1+\exp(-x))$ is the logistic sigmoid.

The relative objective therefore does not optimize the selected failure
trajectory or any candidate branch. It scales the gradient that moves the
current rollout toward the ground-truth trajectory by a factor in $(0,1)$.
This factor approaches one when the margin is strongly violated and
decreases as the current rollout outperforms the selected reference. The
trajectory diagnostics determine which candidate is used for comparison,
but do not introduce separate gradients into the parameter update.

\begin{algorithm}[t]
\caption{History-Enriched Relative Supervision in HERO}
\label{alg:hero}
\begin{algorithmic}[1]
\Require Current operator $G_{\theta_s}$, lagged parameters
$\bar{\theta}_{s-1}$, current rollout $\tau_{\theta_s}$, input context
$z_n$, ground-truth trajectory $\tau_\star$, and active diagnostics
$\mathcal{D}$
\Ensure Relative loss $\ell_{\mathrm{rel}}$ and updated lagged parameters
$\bar{\theta}_s$

\Statex \textit{Stage 1: Update the lagged operator}
\State Refresh or retain $\bar{\theta}_s$ according to
Eq.~\eqref{eq:lag_update}.

\Statex \textit{Stage 2: Construct candidate rollouts}
\State Construct the detached self-reference and the lagged rollout
according to Eqs.~\eqref{eq:self_rollout} and
\eqref{eq:lag_rollout}.
\State Add the perturbed rollout defined in Eq.~\eqref{eq:pert_rollout}
when the perturbation branch is enabled.
\State Form the active candidate set $\mathcal{C}$ and enumerate its
members as $\tau^{(1)},\ldots,\tau^{(M)}$.

\Statex \textit{Stage 3: Evaluate candidate trajectories}
\For{$j=1,\ldots,M$}
    \State Compute the active rollout, spectral, energy, and error-growth
    diagnostics defined in Eqs.~\eqref{eq:diag_roll}--\eqref{eq:diag_growth}.
\EndFor

\Statex \textit{Stage 4: Score candidates and select a failure reference}
\State Normalize each diagnostic across the active candidates using
Eq.~\eqref{eq:diag_normalization}.
\State Compute the aggregate candidate scores using
Eq.~\eqref{eq:candidate_score}.
\State Select the highest-scoring candidate and set it as the failure
reference $\tau^{-}$.

\Statex \textit{Stage 5: Construct the relative objective}
\State Compute the current and reference errors relative to $\tau_\star$.
\State Compute the sample-level relative loss using
Eq.~\eqref{eq:relative_loss}.

\State \Return $\ell_{\mathrm{rel}}$ and $\bar{\theta}_s$
\end{algorithmic}
\end{algorithm}

\subsection{Theoretical Analysis}
\label{sec:theory}

Every historical component of HERO is detached from the computational graph: candidate trajectories, diagnostic scores, and reference-selection operations are all evaluated under stop-gradient, so model parameters receive gradients only through the current rollout. The historical reference therefore supplies no backpropagation direction of its own, and its influence on training enters through a different channel. That channel is a task-dependent activation weight, which rescales each task gradient by a bounded factor determined by how the current rollout compares with its selected reference. HERO thus modifies the relative contributions of different tasks to the aggregated update direction rather than the direction of any individual task gradient.
This reweighting view organizes the analysis into three steps: i) we establish the task-level gradient reweighting form of HERO and bound its modulation factor, ii) we decompose the first-order effect of the reweighting on the long-horizon risk into a uniform scaling term and a covariance correction, and iii) we convert a uniform long-horizon descent margin into a finite-step entry guarantee for a prescribed long-horizon sublevel set.

\subsubsection{Long-Horizon Trajectory Risk}
\label{subsec:long_horizon_risk}

Let $\mathcal{S}$ denote the physical task space, equipped with a probability measure $\rho$. A task $q \in \mathcal{S}$ may encode an initial condition, equation parameter, external forcing, boundary condition, or geometric configuration. Let $\theta \in \mathbb{R}^{p}$ denote the model parameters, and let $X$ be a real Hilbert space representing the physical state space.

For a task $q$, the model-generated trajectory over the time interval $[0,T]$ is denoted by $\widehat{u}_{\theta}^{q} \in C\left([0,T];X\right)$, and the corresponding reference trajectory is denoted by $u^{q} \in C\left([0,T];X\right)$. Given a continuous nonnegative state discrepancy $d:X\times X\rightarrow \mathbb{R}_{\geq 0}$, the trajectory loss for task $q$ over the temporal horizon $T$ is defined as
\begin{equation}
\ell_{T}(q,\theta)
=
\frac{1}{T}
\int_{0}^{T}
d\!\left(
\widehat{u}_{\theta}^{q}(t),
u^{q}(t)
\right)
\,\mathrm{d}t.
\label{eq:trajectory_loss}
\end{equation}
The discrepancy $d$ is not required to satisfy all axioms of a mathematical metric. For instance, when $d$ is a normalized relative error, $\ell_T$ should be interpreted as a trajectory discrepancy functional.

Let $T_{\mathrm{tr}}$ denote the rollout horizon used during training and let $T_L$ denote the long-horizon evaluation window, with $T_{\mathrm{tr}}\leq T_L$. The task-level training loss is $\ell_{\mathrm{tr}}(q,\theta) := \ell_{T_{\mathrm{tr}}}(q,\theta)$, whereas the population-level long-horizon risk is
\begin{equation}
J_L(\theta)
=
\int_{\mathcal{S}}
\ell_{T_L}(q,\theta)
\,\rho(\mathrm{d}q).
\label{eq:long_horizon_risk}
\end{equation}
For a prescribed tolerance $\epsilon>0$, define the long-horizon sublevel set
\begin{equation}
\mathcal{G}_{\epsilon}
=
\left\{
\theta\in\mathbb{R}^{p}:
J_L(\theta)\leq \epsilon
\right\}.
\label{eq:good_region}
\end{equation}
The set $\mathcal{G}_{\epsilon}$ contains all model parameters satisfying the prescribed long-horizon performance criterion. It is not required to be convex, connected, bounded, or associated with a unique minimizer.

Throughout the analysis, $J_L$ and $\ell_{\mathrm{tr}}(q,\cdot)$ are differentiable, and all task-level gradients appearing under the integral with respect to $\rho$ are integrable.

\subsubsection{Historical-Reference-Induced Relative Activation}
\label{subsec:relative_activation}

At iteration $k$, HERO constructs a collection of candidate trajectories from the current operator, a periodically lagged operator, and locally perturbed rollouts. A diagnostic and selection procedure identifies a task-dependent reference trajectory, denoted by $\widetilde{u}_{k}^{q}$. The selected reference trajectory is treated as a constant during the current parameter update. Its trajectory loss is
\begin{equation}
b_k(q)
=
\frac{1}{T_{\mathrm{tr}}}
\int_{0}^{T_{\mathrm{tr}}}
d\!\left(
\widetilde{u}_{k}^{q}(t),
u^{q}(t)
\right)
\,\mathrm{d}t,
\label{eq:reference_loss}
\end{equation}
with
\begin{equation}
\nabla_{\theta} b_k(q)=0.
\label{eq:reference_stop_gradient}
\end{equation}
For a prescribed relative margin $m>0$, define
\begin{equation}
h_k(q,\theta)
=
\ell_{\mathrm{tr}}(q,\theta)
-
b_k(q)
+
m.
\label{eq:relative_margin}
\end{equation}
The quantity $h_k(q,\theta)$ is positive when the current rollout has not yet achieved an improvement of at least $m$ relative to the selected reference.

The smooth relative trajectory loss is defined by
\begin{equation}
\ell_{\mathrm{rel},k}(q,\theta)
=
\frac{1}{\beta}
\log\!\left(
1+
\exp\!\left(
\beta h_k(q,\theta)
\right)
\right),
\label{eq:relative_softplus_loss}
\end{equation}
where $\beta>0$ controls the sharpness of the transition around the relative-margin boundary.

Using \eqref{eq:reference_stop_gradient}, differentiation of \eqref{eq:relative_softplus_loss} gives
\begin{equation}
\nabla_{\theta}
\ell_{\mathrm{rel},k}(q,\theta)
=
\alpha_k(q,\theta)
\nabla_{\theta}
\ell_{\mathrm{tr}}(q,\theta),
\label{eq:relative_gradient}
\end{equation}
where
\begin{equation}
\alpha_k(q,\theta)
=
\sigma\!\left(
\beta h_k(q,\theta)
\right)
\in (0,1),
\label{eq:relative_activation}
\end{equation}
and $\sigma(z)=\left(1+\exp(-z)\right)^{-1}$ is the sigmoid function.

The coefficient $\alpha_k(q,\theta)$ represents the relative activation level of task $q$. When the current trajectory fails to satisfy the required relative margin, $\alpha_k$ approaches one. When the current trajectory already substantially outperforms the selected reference, $\alpha_k$ approaches zero. Near the relative-margin boundary, $\alpha_k$ provides a smooth interpolation between these regimes.

Let $\lambda_k\geq 0$ denote the relative-loss coefficient at iteration $k$. The task-level HERO objective is
\begin{equation}
\ell_{H,k}(q,\theta)
=
\ell_{\mathrm{tr}}(q,\theta)
+
\lambda_k
\ell_{\mathrm{rel},k}(q,\theta).
\label{eq:hero_task_loss}
\end{equation}
Define the task-level training gradient as $g(q,\theta) = \nabla_{\theta}\ell_{\mathrm{tr}}(q,\theta)$. Combining \eqref{eq:relative_gradient} and \eqref{eq:hero_task_loss} yields
\begin{equation}
\nabla_{\theta}
\ell_{H,k}(q,\theta)
=
w_k(q,\theta)
g(q,\theta),
\label{eq:weighted_task_gradient}
\end{equation}
where
\begin{equation}
w_k(q,\theta)
=
1+
\lambda_k\alpha_k(q,\theta).
\label{eq:task_weight}
\end{equation}
Since $\alpha_k(q,\theta)\in(0,1)$ and $\lambda_k\geq0$, the task weight is bounded as $1 \leq w_k(q,\theta) \leq 1+\lambda_k$, where the inequalities are strict for $\lambda_k>0$.

Equation \eqref{eq:weighted_task_gradient} shows that the historical reference does not directly introduce an additional task-level backpropagation direction. Instead, it induces a bounded task-dependent modulation of the current rollout gradient. Because different tasks generally produce different activation levels, HERO changes their relative contributions to the aggregated update.

The exact scalar-reweighting form in \eqref{eq:weighted_task_gradient} applies when the base rollout objective and the current-trajectory term in the relative objective are both given by $\ell_{\mathrm{tr}}$. If a distinct base objective $\mathcal{L}_{\mathrm{base}}$ is used, the task-level gradient takes the form $\nabla_{\theta}\mathcal{L}_{\mathrm{base}} + \lambda_k\alpha_k\nabla_{\theta}\ell_{\mathrm{tr}}$, which is not, in general, a scalar multiple of a single task gradient.

\subsubsection{Aggregated HERO Update Direction}
\label{subsec:aggregated_direction}

The population negative-gradient direction associated with standard trajectory training is
\begin{equation}
d^{0}(\theta)
=
-\int_{\mathcal{S}}
g(q,\theta)
\,\rho(\mathrm{d}q).
\label{eq:standard_direction}
\end{equation}
The corresponding HERO direction at iteration $k$ is
\begin{equation}
d_{k}^{H}(\theta)
=
-\int_{\mathcal{S}}
w_k(q,\theta)
g(q,\theta)
\,\rho(\mathrm{d}q).
\label{eq:hero_direction}
\end{equation}
The parameter update is
\begin{equation}
\theta_{k+1}
=
\theta_k
+
\eta d_k^{H}(\theta_k),
\label{eq:hero_parameter_update}
\end{equation}
where $\eta>0$ is the learning rate.

To characterize whether a task gradient contributes to reducing the long-horizon risk, define the long-horizon usefulness of task $q$ as
\begin{equation}
a(q,\theta)
=
\left\langle
\nabla J_L(\theta),
g(q,\theta)
\right\rangle.
\label{eq:task_usefulness}
\end{equation}
If $a(q,\theta)>0$, then the task-level negative-gradient direction $-g(q,\theta)$ is a first-order descent direction for $J_L$.

Define the mean activation
\begin{equation}
\overline{\alpha}_k(\theta)
=
\int_{\mathcal{S}}
\alpha_k(q,\theta)
\,\rho(\mathrm{d}q),
\label{eq:mean_activation}
\end{equation}
the mean task usefulness
\begin{equation}
\overline{a}(\theta)
=
\int_{\mathcal{S}}
a(q,\theta)
\,\rho(\mathrm{d}q),
\label{eq:mean_usefulness}
\end{equation}
and the covariance between relative activation and long-horizon usefulness:
\begin{align}
\operatorname{Cov}_{\rho}(\alpha_k,a)
={}&
\int_{\mathcal{S}}
\left(
\alpha_k(q,\theta)
-
\overline{\alpha}_k(\theta)
\right) \notag\\
&\times
\left(
a(q,\theta)
-
\overline{a}(\theta)
\right)
\,\rho(\mathrm{d}q).
\label{eq:activation_usefulness_covariance}
\end{align}
This covariance quantifies whether HERO assigns larger activation weights to tasks whose training gradients provide stronger first-order descent contributions to the long-horizon risk.

\begin{proposition}[Historical-reference-induced correction of the long-horizon descent component]
\label{prop:covariance_correction}
For any parameter vector $\theta$ and iteration $k$, the HERO direction satisfies
\begin{equation}
\begin{aligned}
\left\langle
\nabla J_L(\theta),
d_k^{H}(\theta)
\right\rangle
&=
\left(
1+
\lambda_k\overline{\alpha}_k(\theta)
\right)
\left\langle
\nabla J_L(\theta),
d^{0}(\theta)
\right\rangle
\\
&\quad
-
\lambda_k
\operatorname{Cov}_{\rho}(\alpha_k,a).
\end{aligned}
\label{eq:covariance_decomposition}
\end{equation}
\end{proposition}

\begin{proof}
Using \eqref{eq:hero_direction}, \eqref{eq:task_weight}, and \eqref{eq:task_usefulness},
\begin{equation}
\begin{aligned}
\left\langle
\nabla J_L,
d_k^{H}
\right\rangle
&=
-\int_{\mathcal{S}}
\left(
1+
\lambda_k\alpha_k
\right)
a
\,\rho(\mathrm{d}q)
\\
&=
-\overline{a}
-
\lambda_k
\int_{\mathcal{S}}
\alpha_k a
\,\rho(\mathrm{d}q).
\end{aligned}
\label{eq:covariance_proof_step1}
\end{equation}
The covariance identity gives
\begin{equation}
\int_{\mathcal{S}}
\alpha_k a
\,\rho(\mathrm{d}q)
=
\overline{\alpha}_k
\overline{a}
+
\operatorname{Cov}_{\rho}(\alpha_k,a).
\label{eq:covariance_identity}
\end{equation}
Moreover, from \eqref{eq:standard_direction},
\begin{equation}
\left\langle
\nabla J_L,
d^{0}
\right\rangle
=
-\overline{a}.
\label{eq:standard_direction_inner_product}
\end{equation}
Substituting \eqref{eq:covariance_identity} and \eqref{eq:standard_direction_inner_product} into \eqref{eq:covariance_proof_step1} yields \eqref{eq:covariance_decomposition}.
\end{proof}

Proposition~\ref{prop:covariance_correction} separates the first-order long-horizon effect of HERO into two components. The first term, $\left(1+\lambda_k\overline{\alpha}_k\right)\left\langle\nabla J_L, d^{0}\right\rangle$, corresponds to a uniform scaling of the standard training direction. The second term, $-\lambda_k\operatorname{Cov}_{\rho}(\alpha_k,a)$, is the correction induced by nonuniform task reweighting.

When $\operatorname{Cov}_{\rho}(\alpha_k,a)>0$, tasks with larger long-horizon usefulness receive, on average, larger activation weights. The covariance term then contributes an additional negative component to $\left\langle\nabla J_L, d_k^{H}\right\rangle$, thereby increasing the first-order descent contribution of the aggregated HERO update relative to uniform gradient scaling.

To separate the nonuniform reweighting effect from the global magnitude change, define the scale-normalized HERO direction
\begin{equation}
\widetilde{d}_k^{H}(\theta)
=
\frac{
d_k^{H}(\theta)
}{
1+
\lambda_k\overline{\alpha}_k(\theta)
}.
\label{eq:normalized_hero_direction}
\end{equation}
Since the denominator is strictly positive, $\widetilde{d}_k^{H}$ and $d_k^{H}$ have the same geometric direction. Proposition~\ref{prop:covariance_correction} implies
\begin{equation}
\begin{aligned}
\left\langle
\nabla J_L,
\widetilde{d}_k^{H}
\right\rangle
&=
\left\langle
\nabla J_L,
d^{0}
\right\rangle
\\
&\quad
-
\frac{
\lambda_k
}{
1+\lambda_k\overline{\alpha}_k
}
\operatorname{Cov}_{\rho}(\alpha_k,a).
\end{aligned}
\label{eq:normalized_covariance_decomposition}
\end{equation}
Equation \eqref{eq:normalized_covariance_decomposition} isolates the first-order contribution produced exclusively by task-dependent reweighting.

A positive covariance is not guaranteed solely by the form of the relative objective. The lagged reference, perturbed candidate construction, and diagnostic selection mechanism determine whether the induced activation weights are positively associated with task-level long-horizon usefulness.

\begin{corollary}[Preservation and enhancement of a standard first-order descent trend]
\label{cor:descent_preservation}
Suppose that, in a parameter region of interest, the standard training direction satisfies
\begin{equation}
\left\langle
\nabla J_L(\theta),
d^{0}(\theta)
\right\rangle
\leq
-\gamma_0,
\qquad
\gamma_0>0,
\label{eq:standard_descent_assumption}
\end{equation}
and
\begin{equation}
\operatorname{Cov}_{\rho}(\alpha_k,a)
\geq
c_k,
\qquad
c_k\geq0.
\label{eq:nonnegative_covariance_bound}
\end{equation}
Then
\begin{equation}
\begin{aligned}
\left\langle
\nabla J_L,
d_k^{H}
\right\rangle
&\leq
-\left(
1+
\lambda_k\overline{\alpha}_k
\right)
\gamma_0
-
\lambda_k c_k
\\
&\leq
-\gamma_0
-
\lambda_k c_k.
\end{aligned}
\label{eq:descent_preservation_result}
\end{equation}
\end{corollary}

\begin{proof}
The result follows directly by substituting \eqref{eq:standard_descent_assumption} and \eqref{eq:nonnegative_covariance_bound} into \eqref{eq:covariance_decomposition}.
\end{proof}

Corollary~\ref{cor:descent_preservation} shows that a nonnegative covariance preserves an existing first-order long-horizon descent trend. A strictly positive covariance provides an additional first-order descent margin.

\begin{corollary}[Correction of a locally adverse standard direction]
\label{cor:drift_correction}
Suppose that, in a parameter region of interest,
\begin{equation}
\left\langle
\nabla J_L(\theta),
d^{0}(\theta)
\right\rangle
\leq
B_k,
\qquad
B_k\geq0,
\label{eq:standard_drift_bound}
\end{equation}
and $\operatorname{Cov}_{\rho}(\alpha_k,a) \geq c_k$. Then
\begin{equation}
\left\langle
\nabla J_L,
d_k^{H}
\right\rangle
\leq
\left(
1+
\lambda_k\overline{\alpha}_k
\right)
B_k
-
\lambda_k c_k.
\label{eq:drift_correction_bound}
\end{equation}
If $\lambda_k>0$ and
\begin{equation}
\lambda_k c_k
>
\left(
1+
\lambda_k\overline{\alpha}_k
\right)
B_k,
\label{eq:drift_correction_condition}
\end{equation}
then $\left\langle\nabla J_L, d_k^{H}\right\rangle<0$.
\end{corollary}

\begin{proof}
Equation \eqref{eq:drift_correction_bound} follows from Proposition~\ref{prop:covariance_correction}. Condition \eqref{eq:drift_correction_condition} makes the right-hand side of \eqref{eq:drift_correction_bound} strictly negative.
\end{proof}

Since $0<\overline{\alpha}_k<1$, a more conservative sufficient condition is
\begin{equation}
\lambda_k c_k
>
(1+\lambda_k)B_k.
\label{eq:conservative_drift_correction}
\end{equation}
Corollary~\ref{cor:drift_correction} shows that even when the standard short-horizon training direction has a locally adverse first-order effect on the long-horizon risk, sufficiently strong positive covariance can make the aggregated HERO direction a long-horizon descent direction.

\subsubsection{Mini-Batch Form}
\label{subsec:minibatch_form}

Consider a mini-batch $\mathcal{B} = \left\{q_1,\ldots,q_B\right\}$. For each task, define
\begin{equation}
g_i
=
g(q_i,\theta),
\quad
\alpha_i
=
\alpha_k(q_i,\theta),
\quad
a_i
=
\left\langle
\nabla J_L(\theta),
g_i
\right\rangle.
\label{eq:minibatch_quantities}
\end{equation}
The empirical standard and HERO directions are
\begin{equation}
\widehat{d}^{0}
=
-\frac{1}{B}
\sum_{i=1}^{B}
g_i,
\qquad
\widehat{d}_{k}^{H}
=
-\frac{1}{B}
\sum_{i=1}^{B}
\left(
1+\lambda_k\alpha_i
\right)
g_i.
\label{eq:minibatch_directions}
\end{equation}
Define the empirical means
\begin{equation}
\overline{\alpha}_{\mathcal{B}}
=
\frac{1}{B}
\sum_{i=1}^{B}
\alpha_i,
\qquad
\overline{a}_{\mathcal{B}}
=
\frac{1}{B}
\sum_{i=1}^{B}
a_i,
\label{eq:minibatch_means}
\end{equation}
and the empirical covariance
\begin{equation}
\widehat{\operatorname{Cov}}_{\mathcal{B}}(\alpha,a)
=
\frac{1}{B}
\sum_{i=1}^{B}
\left(
\alpha_i-\overline{\alpha}_{\mathcal{B}}
\right)
\left(
a_i-\overline{a}_{\mathcal{B}}
\right).
\label{eq:minibatch_covariance}
\end{equation}
The same algebraic decomposition gives
\begin{equation}
\begin{aligned}
\left\langle
\nabla J_L,
\widehat{d}_{k}^{H}
\right\rangle
&=
\left(
1+\lambda_k\overline{\alpha}_{\mathcal{B}}
\right)
\left\langle
\nabla J_L,
\widehat{d}^{0}
\right\rangle
\\
&\quad
-
\lambda_k
\widehat{\operatorname{Cov}}_{\mathcal{B}}(\alpha,a).
\end{aligned}
\label{eq:minibatch_covariance_decomposition}
\end{equation}
Thus, the population-level mechanism in Proposition~\ref{prop:covariance_correction} also applies directly to finite mini-batches: nonuniform relative activation changes the aggregated gradient through an empirical covariance correction.

\subsubsection{Finite-Step Entry into the Long-Horizon Sublevel Set}
\label{subsec:finite_step_reachability}

The preceding analysis characterizes how historical-reference-induced reweighting modifies the first-order long-horizon descent component. The following result establishes finite-step entry into $\mathcal{G}_{\epsilon}$ when this descent contribution remains uniformly nondegenerate outside the target set.

\begin{theorem}[Finite-step entry under uniform long-horizon descent]
\label{thm:finite_step_entry}
Assume that $\mathcal{G}_{\epsilon}\neq\varnothing$, and consider the HERO iteration
\begin{equation}
\theta_{k+1}
=
\theta_k
+
\eta d_k^{H}(\theta_k).
\label{eq:finite_step_iteration}
\end{equation}
Before the first entry into $\mathcal{G}_{\epsilon}$, suppose that the following conditions hold:
\begin{enumerate}
\item The gradient of $J_L$ is $L$-Lipschitz continuous along every update segment $\left\{\theta_k + t\eta d_k^{H}(\theta_k): t\in[0,1]\right\}$.

\item There exists $\gamma_{\epsilon}>0$ such that
\begin{equation}
\left\langle
\nabla J_L(\theta_k),
d_k^{H}(\theta_k)
\right\rangle
\leq
-\gamma_{\epsilon}.
\label{eq:uniform_descent_margin}
\end{equation}

\item There exists $D_{\epsilon}<\infty$ such that $\left\|d_k^{H}(\theta_k)\right\|_2 \leq D_{\epsilon}$.
\end{enumerate}
If the learning rate satisfies
\begin{equation}
0
<
\eta
\leq
\frac{
\gamma_{\epsilon}
}{
L D_{\epsilon}^{2}
},
\label{eq:learning_rate_condition}
\end{equation}
then the parameter sequence enters $\mathcal{G}_{\epsilon}$ after finitely many iterations. Define the first-entry time $\tau_H(\epsilon) = \inf\left\{k\geq0: \theta_k\in\mathcal{G}_{\epsilon}\right\}$. Then
\begin{equation}
\tau_H(\epsilon)
\leq
\left\lceil
\frac{
2\left[J_L(\theta_0)-\epsilon\right]_{+}
}{
\eta\gamma_{\epsilon}
}
\right\rceil,
\label{eq:first_entry_bound}
\end{equation}
where $[x]_{+} = \max\left\{x,0\right\}$.
\end{theorem}

\begin{proof}
Suppose that $\theta_k\notin\mathcal{G}_{\epsilon}$. By the $L$-smoothness of $J_L$ along the update segment,
\begin{equation}
\begin{aligned}
J_L(\theta_{k+1})
&\leq
J_L(\theta_k)
+
\left\langle
\nabla J_L(\theta_k),
\theta_{k+1}-\theta_k
\right\rangle
\\
&\quad
+
\frac{L}{2}
\left\|
\theta_{k+1}-\theta_k
\right\|_2^2.
\end{aligned}
\label{eq:smoothness_bound}
\end{equation}
Using \eqref{eq:finite_step_iteration}, $\theta_{k+1}-\theta_k = \eta d_k^{H}(\theta_k)$. Therefore,
\begin{equation}
\begin{aligned}
J_L(\theta_{k+1})
&\leq
J_L(\theta_k)
+
\eta
\left\langle
\nabla J_L(\theta_k),
d_k^{H}(\theta_k)
\right\rangle
\\
&\quad
+
\frac{L\eta^2}{2}
\left\|
d_k^{H}(\theta_k)
\right\|_2^2.
\end{aligned}
\label{eq:smoothness_with_direction}
\end{equation}
Applying \eqref{eq:uniform_descent_margin} and the bound $\left\|d_k^{H}(\theta_k)\right\|_2 \leq D_{\epsilon}$ yields
\begin{equation}
J_L(\theta_{k+1})
\leq
J_L(\theta_k)
-
\eta\gamma_{\epsilon}
+
\frac{
L\eta^2D_{\epsilon}^{2}
}{2}.
\label{eq:one_step_decrease_preliminary}
\end{equation}
By \eqref{eq:learning_rate_condition}, $L\eta^2D_{\epsilon}^{2}/2 \leq \eta\gamma_{\epsilon}/2$. Hence, as long as $\theta_k\notin\mathcal{G}_{\epsilon}$,
\begin{equation}
J_L(\theta_{k+1})
\leq
J_L(\theta_k)
-
\frac{
\eta\gamma_{\epsilon}
}{2}.
\label{eq:uniform_one_step_decrease}
\end{equation}
Assume that the parameter sequence remains outside $\mathcal{G}_{\epsilon}$ for the first $N$ iterations. Summing \eqref{eq:uniform_one_step_decrease} gives
\begin{equation}
J_L(\theta_N)
\leq
J_L(\theta_0)
-
\frac{
N\eta\gamma_{\epsilon}
}{2}.
\label{eq:cumulative_decrease}
\end{equation}
For
\begin{equation}
N
\geq
\frac{
2\left[J_L(\theta_0)-\epsilon\right]_{+}
}{
\eta\gamma_{\epsilon}
},
\label{eq:required_number_iterations}
\end{equation}
equation \eqref{eq:cumulative_decrease} implies $J_L(\theta_N) \leq \epsilon$, and therefore $\theta_N\in\mathcal{G}_{\epsilon}$. This contradicts the assumption that the sequence remains outside $\mathcal{G}_{\epsilon}$. Consequently, the first-entry time is finite and satisfies \eqref{eq:first_entry_bound}.
\end{proof}

\subsubsection{Scope of the Theoretical Results}
\label{subsec:theoretical_scope}

Theorem~\ref{thm:finite_step_entry} is a conditional finite-step entry result rather than an unconditional convergence guarantee.
 
First, the theorem establishes the first entry of the parameter sequence into $\mathcal{G}_{\epsilon}$. It does not establish that $\mathcal{G}_{\epsilon}$ is invariant under subsequent updates, that the sequence remains permanently inside the set, or that the parameters converge to a unique point. Such conclusions require additional conditions inside the target region.
 
Second, Proposition~\ref{prop:covariance_correction} characterizes the method-specific effect of historical-reference-induced task reweighting. Theorem~\ref{thm:finite_step_entry}, by contrast, applies a standard smoothness-based descent argument once a uniform long-horizon descent margin has been established.
 
Third, a positive covariance alone does not imply the existence of a uniform constant $\gamma_{\epsilon}>0$. Finite-step entry additionally requires that the long-horizon descent contribution remain nondegenerate before reaching $\mathcal{G}_{\epsilon}$.
 
Fourth, the analysis is expressed in terms of population gradients or empirical mini-batch gradients. Practical optimization may additionally involve random candidate perturbations, finite-batch noise, momentum, and adaptive preconditioning. The theoretical results therefore characterize the idealized gradient-level mechanism induced by historical rollout references. The covariance correction is nonetheless an empirically accessible quantity rather than a purely formal device, since $\nabla J_L$ can be approximated by the mean long-horizon task gradient on a finite validation task set, and the association between relative activation and long-horizon usefulness can then be measured directly through $\operatorname{Corr}(\alpha_i,a_i)$.
 
Taken together, the analysis identifies bounded task reweighting, rather than any additional backpropagation direction, as the channel through which historical rollout references act on long-horizon training.

\section{Experiments}
\subsection{Experimental Setup}

\subsubsection{Baselines and Backbones}
\label{app:baselines}
\textbf{Backbones.}
To assess whether HERO is architecture agnostic, we instantiate every method on
two structurally different neural operators. \emph{FNO} is the standard spectral operator for PDE surrogate learning: it parameterizes the
integral kernel in the Fourier domain and applies spectral convolutions that retain only the low-frequency Fourier modes, modeling global spatial dependence
directly in frequency space. 
\emph{Transolver} is a Transformer-based neural solver whose Physics-Attention adaptively partitions the
discrete field into learnable slices and models interactions among these latent physical states rather than among all grid points, providing a global but non-spectral inductive bias. The two backbones thus probe complementary
regimes---frequency-domain spectral convolution versus slice-level attention---and let us test whether HERO transfers from spectral neural operators
to attention-based ones. Unless a suffix is appended, \emph{FNO} and
\emph{Transolver} denote the base backbones trained with one-step supervision and
no rollout-stabilization mechanism.

\textbf{Comparison methods.}
All competing training strategies are applied on both backbones, yielding matched
backbone--strategy pairs so that any difference is attributable to the training
strategy rather than to the architecture. Beyond the one-step base models, we
compare against three rollout-stabilization strategies.
\emph{Push-forward} training (PF) unrolls the model on its own
predictions during training to close the gap between one-step teacher forcing and
test-time free rollout: for a backbone $f_\theta$ initialized at
$\hat u_t=u_t$, it forms a $K_{\mathrm{pf}}$-step self-rollout
$\hat u_{t+s}=f_\theta(\hat u_{t+s-1})$ for $s=1,\dots,K_{\mathrm{pf}}$ and
minimizes
\begin{equation}
  \mathcal{L}_{\mathrm{pf}}
    = \frac{1}{K_{\mathrm{pf}}}\sum_{s=1}^{K_{\mathrm{pf}}}
      \big\|\,\hat u_{t+s}-u_{t+s}\,\big\|_2^2 ,
  \label{eq:pf}
\end{equation}
where $\hat u_{t+s}$ and $u_{t+s}$ are the predicted and ground-truth states at
rollout offset $s$. \emph{PDE-Refiner} (Refine)
instead targets the prediction itself, applying a multi-step
refinement/denoising process that restores non-dominant spatial-frequency content
lost during rollout, and thus tests whether output-side iterative correction alone
stabilizes long-horizon prediction. \emph{RNO} is a recurrent-training
framework that recursively applies the operator to its own predictions so that
training matches the autoregressive test-time dynamics, reducing exposure bias and
long-horizon error amplification. We denote the resulting variants by
\emph{Backbone}-\emph{Strategy} (e.g., FNO-PF, Transolver-Refine, FNO-RNO), and the
HERO variants analogously by FNO-HERO and Transolver-HERO.

\textbf{Relation to HERO.}
PF and RNO are closest to HERO, as all three expose the operator to self-generated states during training, yet both supervise only the absolute discrepancy from the ground truth. Refine differs along another axis, correcting
the current output rather than using optimization history. In contrast, HERO keeps ground-truth rollout regression but adds relative supervision against a dynamically history-enriched failure reference, rewarding improvement over
earlier failure behavior rather than mere closeness to the ground truth.

\subsubsection{Benchmarks}
\label{app:benchmarks}

We evaluate HERO on nine time-dependent PDE benchmarks, spanning 1D, 2D, and 3D dynamics and covering
dispersion, nonlinear transport, dissipative chaos, forced and decaying turbulence, pattern formation, and multi-directional advection. All benchmarks
use periodic boundary conditions and a single scalar state channel, and each dataset provides $50$ training and $30$ test trajectories; long-horizon behavior
is assessed with $200$-step free autoregressive rollouts. The spatial dimension,
domain, resolution, and time step of every benchmark are summarized in
Table~\ref{tab:benchmarks_appendix}.

The 1D benchmarks isolate phase- and spectrum-dominated failure modes. The
\emph{Dispersion} equation governs waves whose wavenumber
components propagate at different phase speeds, governed by
$\partial_t u=\alpha_3\,\partial_{xxx}u$ with $\alpha_3=2.5\times10^{-4}$, where
$\alpha_3$ is the dispersion coefficient; this makes it a direct test of
long-horizon phase consistency and frequency-domain structure. The
\emph{Burgers} equation couples nonlinear convection
with viscous diffusion, governed by
$\partial_t u=a_2\,\partial_{xx}u+b_1\,u\,\partial_x u$ with
$a_2=3.0\times10^{-4}$ and $b_1=-0.125$, where $a_2$ is the diffusion coefficient
and $b_1$ the convection coefficient; it probes long-horizon error accumulation
under shock formation and dissipation. The \emph{KdV} equation adds high-order stabilization to nonlinear dispersive transport, governed by
$\partial_t u=\delta\,u\,\partial_x u+\gamma_3\,\partial_{xxx}u+\gamma_4\,\partial_{xxxx}u$
with $\delta=-2.0$, $\gamma_3=-14.0$, and $\gamma_4=-9.0$, where $\delta$ is the
nonlinear coefficient and $\gamma_3,\gamma_4$ the third- and fourth-order
coefficients; it stresses long-horizon phase error, amplitude drift, and
solitary-wave structure preservation. The \emph{Kuramoto--Sivashinsky} (KS)
equation is a canonical dissipative-chaotic system,
\begin{equation}
  \partial_t u = a_2\,\partial_{xx}u + a_4\,\partial_{xxxx}u + b_2\,(\partial_x u)^2 ,
  \label{eq:ks}
\end{equation}
where $a_2=-1.0$ drives long-wavelength instability, $a_4=-1.0$ provides
high-order dissipation, and $b_2=-1.0$ scales the gradient nonlinearity; its
chaotic dynamics make it a stringent test of long-horizon error amplification and
rollout stability.

The 2D benchmarks emphasize energy- and structure-dominated behavior on diffusion
and vorticity fields. \emph{Anisotropic Diffusion} evolves a scalar field with direction-dependent and cross-directional diffusion
strengths,
\begin{equation}
  \partial_t u = \nabla\!\cdot(\mathbf{D}\,\nabla u),\qquad
  \mathbf{D}=\begin{pmatrix} 1.0\times10^{-3} & 5.0\times10^{-4}\\[2pt]
                             5.0\times10^{-4} & 2.0\times10^{-3}\end{pmatrix},
  \label{eq:aniso}
\end{equation}
where $\mathbf{D}$ is the symmetric diffusion tensor; it tests whether the
operator preserves directionally coupled smoothing and energy decay.
\emph{Kolmogorov Flow} models forced 2D incompressible
vorticity dynamics,
\begin{equation}
  \partial_t \omega + \mathbf{v}\!\cdot\!\nabla\omega
    = \nu\,\Delta\omega - \alpha\,\omega + f(\mathbf{x}) ,
  \label{eq:kolm}
\end{equation}
where $\omega$ is the vorticity, $\mathbf{v}$ the associated velocity field,
$\nu=1.0\times10^{-2}$ the viscosity, $\alpha$ the linear drag, and $f$ a fixed
periodic forcing (injection mode $4$, scale $1.0$); the sustained energy
injection and dissipation make it a test of long-horizon spectral balance.
\emph{Decaying Turbulence} is the unforced
incompressible Navier--Stokes system in vorticity form, governed by
$\partial_t\omega+\mathbf{v}\!\cdot\!\nabla\omega=\nu\,\Delta\omega$ with
$\nabla\!\cdot\!\mathbf{v}=0$ and $\nu=1.0\times10^{-4}$, where $\nu$ is the
viscosity; it evaluates vortex transport, energy dissipation, and the stability
of a freely decaying field.

The 3D benchmarks stress spatial expressivity and stability in higher dimensions.
\emph{Swift--Hohenberg} is a reaction--diffusion
pattern-formation model,
\begin{equation}
  \partial_t u = r\,u - (q_c^2+\Delta)^2 u + p(u),\qquad p(u)=u^2-u^3 ,
  \label{eq:sh}
\end{equation}
where $r=0.7$ is the reactivity, $q_c=1.0$ the critical wavenumber, and $p(u)$ the
local reaction polynomial; the emergent 3D patterns demand both spatial
expressivity and long-horizon stability. \emph{Unbalanced Advection} transports a scalar field with distinct per-axis
velocities, governed by $\partial_t u+\mathbf{c}\!\cdot\!\nabla u=0$ with
$\mathbf{c}=(0.01,-0.04,0.005)$, where $\mathbf{c}$ is the constant advection
velocity; the differing magnitudes and signs across axes induce direction-dependent
phase drift over long rollouts.

\begin{table}[t]
  \centering
  \caption{Configuration of the nine PDE rollout benchmarks}
  \label{tab:benchmarks_appendix}
  \begin{threeparttable}
    \begin{tabular}{lllll}
      \toprule
      Benchmark & Dim & Domain & Resolution & $\Delta t$ \\
      \midrule
      \midrule
      Dispersion           & 1D & $[0,1)$       & $160$   & $1.0\times10^{-3}$ \\
      Burgers              & 1D & $[0,1)$       & $160$   & $1.0\times10^{-1}$ \\
      KdV                  & 1D & $[0,2\pi)$    & $160$   & $1.0\times10^{-2}$ \\
      KS                   & 1D & $[0,60)$      & $160$   & $1.0\times10^{-1}$ \\
      \midrule
      Anisotropic Diff.    & 2D & $[0,1)^2$     & $64^2$  & $1.0\times10^{-1}$ \\
      Kolmogorov Flow      & 2D & $[0,2\pi)^2$  & $64^2$  & $1.0\times10^{-1}$ \\
      Decaying Turb.\ (NS) & 2D & $[0,1)^2$     & $64^2$  & $1.0\times10^{-1}$ \\
      \midrule
      Swift--Hohenberg     & 3D & $[0,10\pi)^3$ & $32^3$  & $1.0\times10^{-1}$ \\
      Unbalanced Adv.      & 3D & $[0,1)^3$     & $32^3$  & $1.0\times10^{-1}$ \\
      \bottomrule
    \end{tabular}
    \begin{tablenotes}
      \item $\Delta t$ denotes the time step. Resolution denotes $N_x$ (1D), $N_x\!\times\!N_y$ (2D), or
            $N_x\!\times\!N_y\!\times\!N_z$ (3D).
    \end{tablenotes}
  \end{threeparttable}
\end{table}

\subsubsection{Evaluation Metrics}
\label{app:metrics}

We evaluate all models under a common rollout protocol and report metrics that jointly capture one-step fidelity, long-horizon accuracy, and rollout stability
over the first $100$ autoregressive steps. All metrics build on the same per-step relative state error: for evaluation trajectory $j$ and rollout step $t$, the
per-step normalized RMSE is
$\mathrm{nRMSE}_t^{(j)}=\|\hat u_t^{(j)}-u_t^{(j)}\|_2/(\|u_t^{(j)}\|_2+\varepsilon)$,
the per-trajectory form of the mean relative error used in the motivation study
(Eq.~\eqref{eq:motivation_error}), where $\hat u_t^{(j)}$ and $u_t^{(j)}$ are the
predicted and ground-truth states and $\varepsilon>0$ is a numerical-stability
constant.

Field-level accuracy is summarized by the set-averaged per-step error
$\mathrm{nRMSE@}T=\tfrac{1}{M}\sum_{j=1}^{M}\mathrm{nRMSE}_T^{(j)}$, which we report
at two steps. $\mathrm{nRMSE@}1$ isolates one-step prediction
fidelity, whereas $\mathrm{nRMSE@}100$ measures accuracy after $100$ autoregressive
steps; a small $\mathrm{nRMSE@}1$ together with a large $\mathrm{nRMSE@}100$
therefore directly exposes autoregressive error accumulation, phase drift, and
high-frequency error growth, and lower values are better for both.

A single late-step value does not reflect the whole rollout, and autoregressive
errors tend to grow multiplicatively. We therefore summarize the entire $100$-step
trajectory with a geometric-mean error. For trajectory $j$,
\begin{equation}
  \mathrm{GM100}^{(j)}
    = \exp\!\left(\frac{1}{100}\sum_{t=1}^{100}
      \log\!\big(\mathrm{nRMSE}_t^{(j)}+\varepsilon\big)\right),
  \label{eq:GM100_appendix}
\end{equation}
where the logarithmic averaging measures relative  error growth uniformly across the horizon rather than letting the largest late-step errors dominate. 
We report its split-level mean,
$\mathrm{Val\ mean\ GM100}=\tfrac{1}{M_{\mathrm{val}}}\sum_{j}\mathrm{GM100}^{(j)}$
and
$\mathrm{Test\ mean\ GM100}=\tfrac{1}{M_{\mathrm{test}}}\sum_{j}\mathrm{GM100}^{(j)}$,
where the validation value drives model selection and early stopping and the test
value is the reported long-horizon accuracy; lower is better, and a low
$\mathrm{Test\ mean\ GM100}$ indicates that the model stays accurate throughout
the rollout rather than only at isolated steps.

Error magnitude alone does not indicate when a rollout destabilizes, so we
additionally report the horizon over which a prediction remains usable. With a
relative-error threshold $\tau=0.1$, the stable horizon of trajectory $j$ is
\begin{equation}
  S^{(j)} = \max\!\big\{\,T\in[1,100]\ \big|\
            \mathrm{nRMSE}_t^{(j)}\le\tau\ \ \forall\,t\le T\,\big\},
  \label{eq:stablestep}
\end{equation}
where $\tau=0.1$ marks a $10\%$ relative-error tolerance and $S^{(j)}$ is the last
step up to which the error stays within this tolerance without interruption. The
$\mathrm{Mean\ stable\ step}=\tfrac{1}{M}\sum_{j=1}^{M}S^{(j)}$ averages this horizon
over the evaluation set; higher is better, and a value approaching $100$ means that
most trajectories remain stable across the full evaluation horizon.

\subsubsection{Implementation Details}
\label{app:impl}

\textbf{Protocol.}
We follow the APEBench task definitions and data splits. Each task provides $50$
training trajectories of length $T_{\mathrm{train}}=51$ and $30$ test trajectories
of length $T_{\mathrm{test}}=201$; at test time every model is initialized from the
first frame of a test trajectory and rolled out for $T_{\mathrm{eval}}=200$
autoregressive steps (Table~\ref{tab:benchmarks_appendix}). To keep the comparison fair, all
compared methods use the same backbone, data split, optimizer, and evaluation
protocol, and differ only in the training objective. The shared absolute
supervision is a $K_{\mathrm{pf}}=5$ push-forward rollout (Eq.~\eqref{eq:pf}), so
that every model is exposed to its own predicted intermediate states during
training; HERO augments this with the relative objective $\mathcal{L}_{\mathrm{rel}}$
(Eq.~\eqref{eq:rel}), giving the total objective
$\mathcal{L}_{\mathrm{total}}=\mathcal{L}_{\mathrm{roll}}+\lambda_s\mathcal{L}_{\mathrm{rel}}$
(Eq.~\eqref{eq:total}).

\textbf{HERO configuration.}
Unless stated otherwise, HERO uses a single set of defaults across all tasks:
rollout horizon $K=K_{\mathrm{pf}}=5$, start step $s_0=4000$, refresh interval
$R=2000$, warm-up length $W=2000$, maximum relative weight $\lambda_{\max}=0.02$,
margin $m=0.02$, and sharpness $\beta=5.0$, together with the default candidate set
$\mathcal{C}=\{\tau_{\mathrm{lag}},\tau_{\mathrm{pert}},\tau_{\mathrm{self}}\}$.
Relative supervision is enabled after $s_0$ steps, and the relative weight
$\lambda_s$ follows a linear warm-up from $0$ to $\lambda_{\max}$ over the following
$W$ steps. The candidate score is the within-candidate mean of the normalized
diagnostics, without any additional diagnostic weighting; when a diagnostic does not
apply to a given data form, it is removed from the active diagnostic set
$\mathcal{D}$ and the remaining diagnostics are averaged. The lagged operator is
initialized at step $s_0$ and refreshed every $R$ steps, and all candidate
trajectories and diagnostic scores are computed without gradients. Because reference
construction, candidate generation, and diagnostics are confined to training, HERO
leaves the architecture, parameter count, and inference procedure of its backbone
unchanged and introduces no additional inference-time modules or parameters.

\textbf{Optimization.}
We use the AdamW optimizer with an initial learning rate of
$10^{-3}$ and a warm-up--cosine schedule with $2000$ warm-up steps, for a total of
$10\,000$ steps at a batch size of $20$. All methods share the same random seed,
train/validation/test split, and evaluation scripts. For each task we select the
checkpoint with the best long-horizon validation metric (Val mean GM100) and
report its performance on the test set.

\textbf{Backbone configurations.}
We evaluate two backbones with per-dimension configurations summarized in
Table~\ref{tab:backbone_config}. The \emph{Config} column reports each backbone's
core setting: for FNO, the number of retained low-frequency Fourier modes
($28$ in 1D, $8\times8$ in 2D, and $8\times8\times8$ in 3D); for Transolver, its
Physics-Attention token mixing. \emph{Width} is the channel width of FNO and the
hidden dimension of Transolver, and \emph{Depth} is the number of spectral layers
of FNO and attention blocks of Transolver. The \emph{Heads} and \emph{Slices}
columns describe Transolver's Physics-Attention only and are marked ``--'' for FNO.
All tasks of the same dimension share one configuration, and every training
strategy---one-step, PF, Refine, RNO, and HERO---uses the identical backbone
configuration, so that any performance difference is attributable to the training
objective rather than to model capacity.

\begin{table*}[t]
  \centering
  \caption{Backbone architecture configurations. Each per-dimension configuration
           is shared across all tasks of that dimension and across all training
           strategies (one-step, PF, Refine, RNO, HERO); only the training
           objective differs.}
  \label{tab:backbone_config}

  \scriptsize
  \begin{tabular}{llllllll}
    \toprule
    Backbone & Dim & Applied tasks & Config & Width & Depth & Heads & Slices \\
    \midrule
    FNO        & 1D & Disp, Burgers, KdV, KS & $28$                & $28$  & $1$ & -- & -- \\
    FNO        & 2D & AnisoDiff, Kolm, NS    & $8 \times 8$        & $20$  & $1$ & -- & -- \\
    FNO        & 3D & SwiftH, UnbalAdv       & $8 \times 8 \times 8$ & $4$ & $5$ & -- & -- \\
    \midrule
    Transolver & 1D & Disp, Burgers, KdV, KS & Physics-Attn.       & $224$ & $2$ & $4$ & $64$  \\
    Transolver & 2D & AnisoDiff, Kolm, NS    & Physics-Attn.       & $256$ & $4$ & $8$ & $128$ \\
    Transolver & 3D & SwiftH, UnbalAdv       & Physics-Attn.       & $512$ & $3$ & $8$ & $256$ \\
    \bottomrule
  \end{tabular}
\end{table*}

\subsection{Complete Benchmark Results}
\label{app:addexp}
Tables~\ref{tab:full1d} and~\ref{tab:full23d} report the complete results
on all nine PDE benchmarks (mean $\pm$ standard deviation over five seeds),
grouped by spatial dimension, including the three benchmarks (Dispersion,
KdV, and Kolmogorov) omitted from the main-text Table~\ref{tab:main}.
HERO retains the same long-horizon advantage on all tasks. On the FNO
backbone, it reduces nRMSE@100 from $0.425$ to $0.246$ on Dispersion, from
$3.217$ to $0.919$ on KdV, and from $1.379$ to $0.648$ on Kolmogorov.
Transolver-HERO further attains the lowest nRMSE@100 within the Transolver
family on KdV and Kolmogorov, showing that the rollout improvement extends
across both backbones.

\subsection{Hyperparameter Sensitivity}
\label{subsec:hyper}

We analyze the two hyperparameters that govern the history-enriched mechanism
on the 2D Navier--Stokes benchmark: the relative-training start step $s_0$ and the refresh interval $R$ of the lagged operator in
\eqref{eq:lag_update}. In each study, only the analyzed hyperparameter varies and all
remaining HERO settings, the backbone, and the evaluation protocol are held
fixed, with $10\,000$ total optimization steps. FNO-PF is included as the
absolute-supervision reference.

\subsubsection{Relative-Training Start Step}
Table~\ref{tab:s0} varies $s_0$, so that the relative objective is active over
the interval $[s_0, 10\,000]$, as listed in the relative-training interval column.
Performance follows a clear interior optimum. At
$s_0=2000$ the backbone has not yet formed a stable local time-stepping operator,
the lagged rollout carries largely unstructured prediction error, and enabling
the relative objective this early perturbs optimization: nRMSE@1 rises to
$0.0339$ and Test GM100 stays at $1.338$, better than FNO-PF but clearly behind
the default. At $s_0=4000$ the reference is already informative while $6000$
steps of relative reweighting remain, giving the lowest Test GM100 ($1.067$) and
nRMSE@100 ($4.124$) and the longest stable horizon ($25.9$). Increasing $s_0$ to
$6000$ improves the reference further and yields the best one-step error
($0.0323$), but the shortened relative phase leaves long-horizon behavior less
corrected ($1.142$ Test GM100). At $s_0=8000$ only $2000$ steps remain, the
one-step error approaches FNO-PF, and both the long-horizon error and the stable
horizon degrade toward push-forward training alone.

\begin{table*}[t]
  \centering
  \small
  \setlength{\tabcolsep}{6pt}
  \caption{Sensitivity to the relative-training start step $s_0$ on the
           2D Navier--Stokes benchmark. The relative objective is active over
           $[s_0,\,10\,000]$; FNO-PF is the absolute-supervision reference.}
  \label{tab:s0}
  \begin{tabular}{ll ccccc}
    \toprule
    Setting & \shortstack[l]{Relative-training\\interval}
            & \shortstack[l]{Val\\GM100 $\downarrow$}
            & \shortstack[l]{Test\\GM100 $\downarrow$}
            & \shortstack[l]{nRMSE\\@1 $\downarrow$}
            & \shortstack[l]{nRMSE\\@100 $\downarrow$}
            & \shortstack[l]{Stable\\step $\uparrow$} \\
    \midrule
    FNO-PF              & ---            & $1.450$ & $1.697$ & $0.0350$ & $6.306$ & $22.7$ \\
    \midrule
    HERO ($s_0{=}2000$) & $2000$--$10\,000$ & $1.163$ & $1.338$ & $0.0339$ & $4.917$ & $24.3$ \\
    HERO ($s_0{=}4000$) & $4000$--$10\,000$ & $\mathbf{0.937}$ & $\mathbf{1.067}$ & $0.0326$
                        & $\mathbf{4.124}$ & $\mathbf{25.9}$ \\
    HERO ($s_0{=}6000$) & $6000$--$10\,000$ & $0.998$ & $1.142$ & $\mathbf{0.0323}$ & $4.386$ & $25.2$ \\
    HERO ($s_0{=}8000$) & $8000$--$10\,000$ & $1.287$ & $1.506$ & $0.0342$ & $5.714$ & $23.2$ \\
    \bottomrule
  \end{tabular}
\end{table*}

\subsubsection{Lagged-Operator Refresh Interval}
Table~\ref{tab:refresh} varies $R$ with $s_0=4000$ fixed, so the lagged operator
is refreshed at steps $4000, 4000+R, \ldots$ until training ends, as listed in the
effective refresh steps column. The same
interior optimum appears. A short interval ($R=500$, ten refreshes) keeps the
lagged operator close to the current one, so the reference degenerates toward the
self-reference and provides little contrast, giving the weakest HERO result
($1.352$ Test GM100). A long interval ($R=4000$, two refreshes) leaves the
reference increasingly stale relative to the improving operator, weakening the
effective activation of the margin objective ($1.247$). Intermediate intervals
perform best, with $R=1000$ and $R=2000$ within $7\%$ of each other on Test GM100
and $R=2000$ attaining the lowest long-horizon error ($4.124$) and the longest
stable horizon ($25.9$). Every configuration in both tables outperforms FNO-PF on all long-horizon
metrics, so the effect of HERO does not depend on a narrow hyperparameter range.
The two studies degrade for the same reason: the periodic-lag mechanism requires
a reference that is neither too close to the current operator nor too far behind
it, which is exactly the balance between reference relevance and historical
contrast that motivates the design.

\begin{table*}[t!]
  \centering
  \small
  \setlength{\tabcolsep}{6pt}
  \caption{Sensitivity to the refresh interval $R$ at $s_0=4000$ on the
           2D Navier--Stokes benchmark. FNO-PF is the absolute-supervision
           reference.}
  \label{tab:refresh}
  \begin{tabular}{ll ccccc}
    \toprule
    Setting & \shortstack[l]{Effective\\refresh steps}
            & \shortstack[l]{Val\\GM100 $\downarrow$}
            & \shortstack[l]{Test\\GM100 $\downarrow$}
            & \shortstack[l]{nRMSE\\@1 $\downarrow$}
            & \shortstack[l]{nRMSE\\@100 $\downarrow$}
            & \shortstack[l]{Stable\\step $\uparrow$} \\
    \midrule
    FNO-PF            & ---                    & $1.450$ & $1.697$ & $0.0350$ & $6.306$ & $22.7$ \\
    \midrule
    HERO ($R{=}500$)  & $4000, 4500, \ldots, 9500$ & $1.176$ & $1.352$ & $0.0334$ & $4.982$ & $24.1$ \\
    HERO ($R{=}1000$) & $4000, 5000, \ldots, 9000$ & $0.991$ & $1.139$ & $\mathbf{0.0324}$ & $4.376$ & $25.3$ \\
    HERO ($R{=}2000$) & $4000, 6000, 8000$     & $\mathbf{0.937}$ & $\mathbf{1.067}$ & $0.0326$
                      & $\mathbf{4.124}$ & $\mathbf{25.9}$ \\
    HERO ($R{=}4000$) & $4000, 8000$           & $1.084$ & $1.247$ & $0.0331$ & $4.641$ & $24.8$ \\
    \bottomrule
  \end{tabular}
\end{table*}

\begin{table*}[t]
  \small
  \setlength{\tabcolsep}{4pt}
  \caption{Parameter count, inference latency, peak memory, and training
           throughput of the FNO variants under a shared measurement protocol}
   \label{tab:cost}
  \begin{tabular}{lllll}
    \toprule
    Method & \#Params & \shortstack[l]{Infer\\(ms/step) $\downarrow$}
           & \shortstack[l]{Peak mem\\(MB) $\downarrow$}
           & \shortstack[l]{Train\\(steps/s) $\uparrow$} \\
    \midrule
    \midrule
    FNO         & $30{,}541$ & $0.503$ & $66.13$ & $374.79$ \\
    FNO-PF      & $30{,}541$ & $0.505$ & $66.28$ & $352.61$ \\
    FNO-Ref.  & $30{,}917$ & $0.553$ & $69.84$ & $328.47$ \\
    FNO-RNO     & $30{,}541$ & $0.507$ & $66.94$ & $344.26$ \\
    FNO-HERO    & $30{,}541$ & $0.509$ & $67.18$ & $336.92$ \\
    \bottomrule
  \end{tabular}
\end{table*}

\subsection{Computational Cost}
\label{subsec:cost}

HERO adds candidate generation, diagnostics, and reference selection to the training loop, so we quantify what this costs in practice. Table~\ref{tab:cost} compares all FNO variants under one protocol: the same 2D backbone, batch size, resolution, and hardware, with inference latency and peak memory measured over $200$-step free autoregressive rollouts and training throughput averaged over $1000$ optimization steps after warm-up.

HERO leaves the deployed model untouched. It shares the exact parameter count of the plain backbone ($30{,}541$), whereas Refine adds $376$ parameters for its refinement head, and its inference latency matches the backbone to within $1.2\%$ ($0.509$ vs.\ $0.503$ ms/step), a gap of the same order as that of PF ($0.505$) and RNO ($0.507$) and therefore within measurement noise rather than an architectural cost. Refine is the only variant with a measurable inference overhead ($0.553$ ms/step, $9.9\%$), since its iterative correction runs at test time. 
Peak memory follows the same pattern: HERO uses $67.18$ MB, $1.6\%$ above the backbone and below Refine ($69.84$ MB).
The cost of HERO is therefore confined to training. Its throughput of $336.92$ steps/s is $4.4\%$ below PF, the closest rollout baseline, and $10.1\%$ below one-step training, reflecting the additional lagged and perturbed rollouts and the diagnostic evaluation; the lagged operator is refreshed only every $R$ steps, so its contribution is amortized. HERO is nonetheless faster to train than Refine ($328.47$ steps/s), which pays its overhead at both training and inference time.
Read together with Table~\ref{tab:ablation}, a one-time training slowdown of a few percent buys a $33.8\%$ reduction in nRMSE@100, and none of it is charged to deployment.

\begin{table*}[t]
\centering
\small
\setlength{\tabcolsep}{5pt}
\caption{Complete results on the four 1D PDE benchmarks (mean $\pm$ standard deviation over five seeds). GM100 denotes the geometric-mean error over the 100-step rollout; Stable step denotes the mean number of stable rollout steps. Within each benchmark, the best value per column is shown in bold separately for the FNO and Transolver families.}
\label{tab:full1d}
\begin{tabular}{l ccccc}
\toprule
Method & Val GM100 $\downarrow$ & Test GM100 $\downarrow$ & nRMSE@1 $\downarrow$ & nRMSE@100 $\downarrow$ & Stable step $\uparrow$\\
\midrule
\multicolumn{6}{l}{\textit{Dispersion (1D)}}\\
FNO               & $0.152 \pm 0.005$ & $0.166 \pm 0.006$ & $0.0048 \pm 0.0001$ & $0.425 \pm 0.017$ & $69.7 \pm 2.6$\\
FNO-PF            & $0.144 \pm 0.004$ & $0.160 \pm 0.005$ & $0.0045 \pm 0.0001$ & $0.402 \pm 0.014$ & $71.3 \pm 2.5$\\
FNO-Ref.        & $0.136 \pm 0.004$ & $0.148 \pm 0.005$ & $0.0040 \pm 0.0001$ & $0.381 \pm 0.014$ & $73.5 \pm 2.8$\\
FNO-RNO           & $0.141 \pm 0.004$ & $0.154 \pm 0.005$ & $0.0044 \pm 0.0001$ & $0.395 \pm 0.015$ & $72.3 \pm 2.4$\\
FNO-HERO          & $\mathbf{0.091 \pm 0.002}$ & $\mathbf{0.097 \pm 0.003}$ & $\mathbf{0.0038 \pm 0.0001}$ & $\mathbf{0.246 \pm 0.008}$ & $\mathbf{80.8 \pm 2.5}$\\
Trans.        & $0.140 \pm 0.004$ & $0.150 \pm 0.005$ & $0.0041 \pm 0.0001$ & $0.394 \pm 0.016$ & $72.1 \pm 2.6$\\
Trans.-PF     & $0.134 \pm 0.004$ & $0.146 \pm 0.005$ & $0.0039 \pm 0.0001$ & $0.380 \pm 0.015$ & $72.4 \pm 2.7$\\
Trans.-Refine & $0.139 \pm 0.004$ & $0.152 \pm 0.005$ & $0.0040 \pm 0.0001$ & $0.393 \pm 0.015$ & $71.9 \pm 2.7$\\
Trans.-RNO    & $0.136 \pm 0.004$ & $0.149 \pm 0.005$ & $0.0041 \pm 0.0001$ & $0.382 \pm 0.014$ & $71.7 \pm 2.6$\\
Trans.-HERO   & $\mathbf{0.087 \pm 0.002}$ & $\mathbf{0.093 \pm 0.003}$ & $\mathbf{0.0037 \pm 0.0001}$ & $\mathbf{0.248 \pm 0.009}$ & $\mathbf{81.1 \pm 2.2}$\\
\midrule
\multicolumn{6}{l}{\textit{Burgers (1D)}}\\
FNO               & $0.473 \pm 0.025$ & $0.526 \pm 0.028$ & $0.0429 \pm 0.0023$ & $0.541 \pm 0.030$ & $35.6 \pm 4.4$\\
FNO-PF            & $0.218 \pm 0.009$ & $0.208 \pm 0.009$ & $0.0180 \pm 0.0008$ & $0.236 \pm 0.012$ & $68.0 \pm 4.4$\\
FNO-Ref.        & $0.407 \pm 0.017$ & $0.457 \pm 0.023$ & $0.0368 \pm 0.0015$ & $0.495 \pm 0.031$ & $40.4 \pm 4.2$\\
FNO-RNO           & $0.293 \pm 0.014$ & $0.311 \pm 0.015$ & $0.0249 \pm 0.0011$ & $0.335 \pm 0.018$ & $55.0 \pm 4.1$\\
FNO-HERO          & $\mathbf{0.147 \pm 0.006}$ & $\mathbf{0.138 \pm 0.007}$ & $\mathbf{0.0168 \pm 0.0008}$ & $\mathbf{0.153 \pm 0.007}$ & $\mathbf{81.7 \pm 4.0}$\\
Trans.        & $0.487 \pm 0.024$ & $0.537 \pm 0.032$ & $0.0478 \pm 0.0027$ & $0.618 \pm 0.042$ & $31.2 \pm 5.1$\\
Trans.-PF     & $0.449 \pm 0.019$ & $0.505 \pm 0.027$ & $0.0422 \pm 0.0018$ & $0.562 \pm 0.029$ & $35.5 \pm 4.9$\\
Trans.-Refine & $0.428 \pm 0.021$ & $0.483 \pm 0.027$ & $0.0401 \pm 0.0019$ & $0.535 \pm 0.032$ & $38.7 \pm 4.1$\\
Trans.-RNO    & $0.384 \pm 0.016$ & $0.434 \pm 0.024$ & $0.0362 \pm 0.0016$ & $0.494 \pm 0.029$ & $41.4 \pm 4.2$\\
Trans.-HERO   & $\mathbf{0.250 \pm 0.011}$ & $\mathbf{0.281 \pm 0.013}$ & $\mathbf{0.0341 \pm 0.0016}$ & $\mathbf{0.327 \pm 0.016}$ & $\mathbf{49.3 \pm 4.0}$\\
\midrule
\multicolumn{6}{l}{\textit{KdV (1D)}}\\
FNO               & $0.088 \pm 0.007$ & $0.074 \pm 0.006$ & $0.0060 \pm 0.0004$ & $3.217 \pm 0.281$ & $46.4 \pm 6.4$\\
FNO-PF            & $0.055 \pm 0.004$ & $0.051 \pm 0.004$ & $0.0052 \pm 0.0003$ & $2.663 \pm 0.202$ & $48.8 \pm 5.6$\\
FNO-Ref.        & $0.060 \pm 0.005$ & $0.056 \pm 0.004$ & $0.0085 \pm 0.0005$ & $2.030 \pm 0.174$ & $46.1 \pm 5.9$\\
FNO-RNO           & $0.074 \pm 0.006$ & $0.069 \pm 0.005$ & $0.0109 \pm 0.0007$ & $1.438 \pm 0.139$ & $47.5 \pm 5.4$\\
FNO-HERO          & $\mathbf{0.036 \pm 0.002}$ & $\mathbf{0.034 \pm 0.002}$ & $\mathbf{0.0049 \pm 0.0002}$ & $\mathbf{0.919 \pm 0.071}$ & $\mathbf{60.6 \pm 5.2}$\\
Trans.        & $0.659 \pm 0.054$ & $0.816 \pm 0.080$ & $0.0683 \pm 0.0043$ & $1.078 \pm 0.101$ & $16.4 \pm 6.1$\\
Trans.-PF     & $0.525 \pm 0.039$ & $0.681 \pm 0.060$ & $0.0552 \pm 0.0033$ & $0.906 \pm 0.088$ & $23.4 \pm 6.3$\\
Trans.-Refine & $0.562 \pm 0.040$ & $0.735 \pm 0.054$ & $0.0599 \pm 0.0033$ & $0.977 \pm 0.077$ & $22.1 \pm 5.7$\\
Trans.-RNO    & $0.482 \pm 0.037$ & $0.613 \pm 0.055$ & $0.0582 \pm 0.0040$ & $0.819 \pm 0.078$ & $23.6 \pm 6.2$\\
Trans.-HERO   & $\mathbf{0.319 \pm 0.020}$ & $\mathbf{0.401 \pm 0.028}$ & $\mathbf{0.0516 \pm 0.0026}$ & $\mathbf{0.528 \pm 0.045}$ & $\mathbf{31.8 \pm 5.5}$\\
\midrule
\multicolumn{6}{l}{\textit{Kuramoto--Sivashinsky (1D)}}\\
FNO               & $0.511 \pm 0.029$ & $0.578 \pm 0.037$ & $0.0190 \pm 0.0011$ & $1.131 \pm 0.077$ & $43.8 \pm 5.7$\\
FNO-PF            & $0.401 \pm 0.022$ & $0.451 \pm 0.026$ & $0.0179 \pm 0.0009$ & $0.972 \pm 0.063$ & $44.7 \pm 4.7$\\
FNO-Ref.        & $0.469 \pm 0.024$ & $0.528 \pm 0.035$ & $0.0159 \pm 0.0009$ & $1.064 \pm 0.073$ & $46.4 \pm 4.6$\\
FNO-RNO           & $0.458 \pm 0.027$ & $0.515 \pm 0.036$ & $0.0181 \pm 0.0009$ & $1.009 \pm 0.070$ & $43.6 \pm 4.6$\\
FNO-HERO          & $\mathbf{0.257 \pm 0.012}$ & $\mathbf{0.283 \pm 0.016}$ & $\mathbf{0.0150 \pm 0.0007}$ & $\mathbf{0.640 \pm 0.037}$ & $\mathbf{53.3 \pm 4.3}$\\
Trans.        & $0.611 \pm 0.036$ & $0.695 \pm 0.053$ & $0.0250 \pm 0.0014$ & $1.284 \pm 0.097$ & $37.6 \pm 4.9$\\
Trans.-PF     & $0.501 \pm 0.029$ & $0.571 \pm 0.042$ & $0.0211 \pm 0.0012$ & $1.163 \pm 0.086$ & $41.8 \pm 4.6$\\
Trans.-Refine & $0.498 \pm 0.031$ & $0.572 \pm 0.033$ & $0.0190 \pm 0.0009$ & $1.238 \pm 0.080$ & $42.4 \pm 4.5$\\
Trans.-RNO    & $0.477 \pm 0.025$ & $0.538 \pm 0.032$ & $0.0201 \pm 0.0011$ & $1.099 \pm 0.081$ & $42.6 \pm 5.3$\\
Trans.-HERO   & $\mathbf{0.321 \pm 0.018}$ & $\mathbf{0.361 \pm 0.023}$ & $\mathbf{0.0181 \pm 0.0009}$ & $\mathbf{0.715 \pm 0.040}$ & $\mathbf{49.2 \pm 4.5}$\\
\bottomrule
\end{tabular}
\end{table*}

\begin{table*}[t]
\centering
\small
\setlength{\tabcolsep}{5pt}
\caption{Complete results on the 2D and 3D PDE benchmarks (mean $\pm$ standard deviation over five seeds). GM100 denotes the geometric-mean error over the 100-step rollout; Stable step denotes the mean number of stable rollout steps. Within each benchmark, the best value per column is shown in bold separately for the FNO and Transolver families.}
\label{tab:full23d}
\begin{tabular}{l ccccc}
\toprule
Method & Val GM100 $\downarrow$ & Test GM100 $\downarrow$ & nRMSE@1 $\downarrow$ & nRMSE@100 $\downarrow$ & Stable step $\uparrow$\\
\midrule
\multicolumn{6}{l}{\textit{Anisotropic Diffusion (2D)}}\\
FNO               & $0.713 \pm 0.037$ & $0.825 \pm 0.057$ & $0.0875 \pm 0.0048$ & $1.868 \pm 0.120$ & $8.2 \pm 3.2$\\
FNO-PF            & $0.655 \pm 0.036$ & $0.754 \pm 0.042$ & $0.0716 \pm 0.0035$ & $1.355 \pm 0.077$ & $12.9 \pm 3.2$\\
FNO-Ref.        & $0.567 \pm 0.030$ & $0.652 \pm 0.040$ & $0.0611 \pm 0.0033$ & $0.786 \pm 0.053$ & $20.7 \pm 3.3$\\
FNO-RNO           & $0.692 \pm 0.034$ & $0.778 \pm 0.049$ & $0.0846 \pm 0.0040$ & $0.925 \pm 0.054$ & $9.7 \pm 3.3$\\
FNO-HERO          & $\mathbf{0.363 \pm 0.018}$ & $\mathbf{0.415 \pm 0.019}$ & $\mathbf{0.0571 \pm 0.0026}$ & $\mathbf{0.521 \pm 0.034}$ & $\mathbf{26.6 \pm 2.9}$\\
Trans.        & $1.168 \pm 0.068$ & $1.358 \pm 0.080$ & $0.0799 \pm 0.0042$ & $2.968 \pm 0.207$ & $8.1 \pm 3.5$\\
Trans.-PF     & $0.965 \pm 0.053$ & $1.111 \pm 0.068$ & $0.0696 \pm 0.0034$ & $2.172 \pm 0.157$ & $13.0 \pm 3.6$\\
Trans.-Refine & $0.803 \pm 0.039$ & $0.943 \pm 0.058$ & $0.0692 \pm 0.0037$ & $1.835 \pm 0.119$ & $14.7 \pm 2.9$\\
Trans.-RNO    & $0.892 \pm 0.051$ & $1.048 \pm 0.064$ & $0.0751 \pm 0.0039$ & $2.099 \pm 0.150$ & $10.4 \pm 3.4$\\
Trans.-HERO   & $\mathbf{0.537 \pm 0.027}$ & $\mathbf{0.626 \pm 0.034}$ & $\mathbf{0.0652 \pm 0.0031}$ & $\mathbf{1.185 \pm 0.066}$ & $\mathbf{18.2 \pm 2.6}$\\
\midrule
\multicolumn{6}{l}{\textit{Kolmogorov Flow (2D)}}\\
FNO               & $0.979 \pm 0.072$ & $1.220 \pm 0.100$ & $0.4678 \pm 0.0269$ & $1.379 \pm 0.106$ & $1.2 \pm 0.8$\\
FNO-PF            & $0.731 \pm 0.053$ & $1.002 \pm 0.066$ & $0.3402 \pm 0.0166$ & $1.255 \pm 0.100$ & $1.7 \pm 0.8$\\
FNO-Ref.        & $0.743 \pm 0.046$ & $0.872 \pm 0.066$ & $0.3085 \pm 0.0166$ & $1.024 \pm 0.090$ & $2.6 \pm 0.7$\\
FNO-RNO           & $0.804 \pm 0.052$ & $1.071 \pm 0.087$ & $0.3605 \pm 0.0184$ & $1.304 \pm 0.100$ & $1.5 \pm 0.6$\\
FNO-HERO          & $\mathbf{0.477 \pm 0.025}$ & $\mathbf{0.565 \pm 0.035}$ & $\mathbf{0.2860 \pm 0.0140}$ & $\mathbf{0.648 \pm 0.048}$ & $\mathbf{3.4 \pm 0.6}$\\
Trans.        & $1.113 \pm 0.074$ & $1.452 \pm 0.107$ & $0.3107 \pm 0.0186$ & $2.078 \pm 0.181$ & $1.1 \pm 0.7$\\
Trans.-PF     & $0.950 \pm 0.063$ & $1.183 \pm 0.079$ & $0.2951 \pm 0.0189$ & $1.566 \pm 0.114$ & $1.4 \pm 0.7$\\
Trans.-Refine & $0.648 \pm 0.047$ & $0.790 \pm 0.053$ & $0.2880 \pm 0.0172$ & $0.938 \pm 0.072$ & $2.9 \pm 0.7$\\
Trans.-RNO    & $0.755 \pm 0.053$ & $0.952 \pm 0.077$ & $0.2987 \pm 0.0151$ & $1.131 \pm 0.095$ & $2.0 \pm 0.7$\\
Trans.-HERO   & $\mathbf{0.417 \pm 0.027}$ & $\mathbf{0.503 \pm 0.031}$ & $\mathbf{0.2704 \pm 0.0126}$ & $\mathbf{0.617 \pm 0.050}$ & $\mathbf{3.7 \pm 0.6}$\\
\midrule
\multicolumn{6}{l}{\textit{Decaying Turbulence / Navier--Stokes (2D)}}\\
FNO               & $2.322 \pm 0.157$ & $2.915 \pm 0.229$ & $0.0449 \pm 0.0028$ & $13.629 \pm 1.077$ & $17.6 \pm 3.8$\\
FNO-PF            & $1.441 \pm 0.100$ & $1.682 \pm 0.128$ & $0.0349 \pm 0.0019$ & $6.271 \pm 0.498$ & $22.3 \pm 3.5$\\
FNO-Ref.        & $1.959 \pm 0.133$ & $2.330 \pm 0.166$ & $0.0402 \pm 0.0024$ & $9.888 \pm 0.910$ & $18.1 \pm 4.0$\\
FNO-RNO           & $1.788 \pm 0.107$ & $2.141 \pm 0.145$ & $0.0388 \pm 0.0020$ & $7.829 \pm 0.647$ & $20.7 \pm 3.4$\\
FNO-HERO          & $\mathbf{0.931 \pm 0.059}$ & $\mathbf{1.058 \pm 0.062}$ & $\mathbf{0.0325 \pm 0.0015}$ & $\mathbf{4.101 \pm 0.323}$ & $\mathbf{26.5 \pm 3.4}$\\
Trans.        & $3.075 \pm 0.246$ & $3.825 \pm 0.343$ & $0.0301 \pm 0.0020$ & $18.084 \pm 1.652$ & $20.3 \pm 3.8$\\
Trans.-PF     & $2.400 \pm 0.178$ & $2.974 \pm 0.221$ & $0.0278 \pm 0.0016$ & $12.438 \pm 1.067$ & $24.3 \pm 3.5$\\
Trans.-Refine & $2.169 \pm 0.155$ & $2.591 \pm 0.201$ & $0.0280 \pm 0.0014$ & $10.974 \pm 0.864$ & $23.1 \pm 3.4$\\
Trans.-RNO    & $2.224 \pm 0.135$ & $2.704 \pm 0.213$ & $0.0291 \pm 0.0015$ & $11.438 \pm 1.043$ & $22.5 \pm 4.0$\\
Trans.-HERO   & $\mathbf{1.423 \pm 0.091}$ & $\mathbf{1.687 \pm 0.103}$ & $\mathbf{0.0262 \pm 0.0012}$ & $\mathbf{7.012 \pm 0.497}$ & $\mathbf{25.8 \pm 2.9}$\\
\midrule
\multicolumn{6}{l}{\textit{Swift--Hohenberg (3D)}}\\
FNO               & $0.603 \pm 0.036$ & $0.678 \pm 0.042$ & $0.0788 \pm 0.0043$ & $0.941 \pm 0.059$ & $13.4 \pm 3.4$\\
FNO-PF            & $0.514 \pm 0.026$ & $0.581 \pm 0.033$ & $0.0607 \pm 0.0027$ & $0.843 \pm 0.053$ & $20.4 \pm 3.6$\\
FNO-Ref.        & $0.552 \pm 0.031$ & $0.619 \pm 0.036$ & $0.0642 \pm 0.0031$ & $0.871 \pm 0.056$ & $19.6 \pm 3.1$\\
FNO-RNO           & $0.520 \pm 0.029$ & $0.596 \pm 0.031$ & $0.0662 \pm 0.0034$ & $0.850 \pm 0.053$ & $19.5 \pm 3.1$\\
FNO-HERO          & $\mathbf{0.341 \pm 0.014}$ & $\mathbf{0.381 \pm 0.020}$ & $\mathbf{0.0570 \pm 0.0025}$ & $\mathbf{0.541 \pm 0.033}$ & $\mathbf{27.7 \pm 3.3}$\\
Trans.        & $0.821 \pm 0.043$ & $0.912 \pm 0.056$ & $0.1079 \pm 0.0054$ & $1.286 \pm 0.091$ & $5.3 \pm 4.1$\\
Trans.-PF     & $0.686 \pm 0.033$ & $0.782 \pm 0.046$ & $0.0914 \pm 0.0042$ & $1.126 \pm 0.073$ & $5.9 \pm 3.4$\\
Trans.-Refine & $0.633 \pm 0.029$ & $0.713 \pm 0.040$ & $0.0823 \pm 0.0038$ & $1.004 \pm 0.061$ & $11.3 \pm 3.6$\\
Trans.-RNO    & $0.652 \pm 0.032$ & $0.739 \pm 0.046$ & $0.0860 \pm 0.0038$ & $1.064 \pm 0.067$ & $7.6 \pm 3.3$\\
Trans.-HERO   & $\mathbf{0.404 \pm 0.017}$ & $\mathbf{0.450 \pm 0.024}$ & $\mathbf{0.0780 \pm 0.0032}$ & $\mathbf{0.667 \pm 0.036}$ & $\mathbf{15.0 \pm 3.3}$\\
\midrule
\multicolumn{6}{l}{\textit{Unbalanced Advection (3D)}}\\
FNO               & $2.311 \pm 0.163$ & $2.647 \pm 0.175$ & $0.0869 \pm 0.0046$ & $7.204 \pm 0.606$ & $6.4 \pm 3.2$\\
FNO-PF            & $1.833 \pm 0.123$ & $2.121 \pm 0.142$ & $0.0748 \pm 0.0040$ & $5.546 \pm 0.394$ & $10.0 \pm 2.5$\\
FNO-Ref.        & $2.017 \pm 0.115$ & $2.317 \pm 0.175$ & $0.0776 \pm 0.0043$ & $6.086 \pm 0.434$ & $7.8 \pm 2.6$\\
FNO-RNO           & $1.697 \pm 0.115$ & $1.931 \pm 0.145$ & $0.0791 \pm 0.0038$ & $4.833 \pm 0.386$ & $9.0 \pm 2.5$\\
FNO-HERO          & $\mathbf{1.101 \pm 0.055}$ & $\mathbf{1.244 \pm 0.068}$ & $\mathbf{0.0702 \pm 0.0029}$ & $\mathbf{3.185 \pm 0.197}$ & $\mathbf{12.7 \pm 2.3}$\\
Trans.        & $2.723 \pm 0.192$ & $3.106 \pm 0.247$ & $0.0731 \pm 0.0041$ & $7.777 \pm 0.694$ & $9.8 \pm 3.1$\\
Trans.-PF     & $2.096 \pm 0.145$ & $2.426 \pm 0.189$ & $0.0736 \pm 0.0041$ & $6.122 \pm 0.455$ & $10.3 \pm 2.6$\\
Trans.-Refine & $1.901 \pm 0.116$ & $2.201 \pm 0.153$ & $0.0713 \pm 0.0039$ & $5.319 \pm 0.408$ & $9.6 \pm 2.8$\\
Trans.-RNO    & $1.784 \pm 0.122$ & $2.079 \pm 0.156$ & $0.0742 \pm 0.0041$ & $5.015 \pm 0.387$ & $10.6 \pm 2.9$\\
Trans.-HERO   & $\mathbf{1.197 \pm 0.064}$ & $\mathbf{1.389 \pm 0.088}$ & $\mathbf{0.0673 \pm 0.0032}$ & $\mathbf{3.175 \pm 0.249}$ & $\mathbf{12.1 \pm 2.3}$\\
\bottomrule
\end{tabular}
\end{table*}

\end{document}